\documentclass[conference]{IEEEtran}
\IEEEoverridecommandlockouts
\usepackage{cite}
\usepackage{amsmath,amssymb,amsfonts}
\usepackage{amsthm}
\usepackage{rotating}
\usepackage{algorithm}
\usepackage{algpseudocode}
\usepackage{threeparttable}
\usepackage{booktabs}
\usepackage{graphicx}
\usepackage{textcomp}
\usepackage{xcolor}
\usepackage{multirow}
\usepackage{pifont}
\usepackage[normalem]{ulem} 

\makeatletter 
\newcommand\redsout{\bgroup\markoverwith{\textcolor{red}{\rule[0.5ex]{2pt}{0.4pt}}}\ULon}
\makeatother

\newcommand{\cmark}{\ding{51}}  
\newcommand{\xmark}{\ding{55}}  
\newtheorem{theorem}{Theorem}
\newtheorem{definition}{Definition}
\newtheorem{proposition}{Proposition}
\newtheorem{corollary}{Corollary}
\newtheorem{remark}{Remark}
\usepackage{subcaption}

\usepackage{flushend}
\usepackage{balance}

\begin{document}

\title{{\footnotesize This is the preprint of the work
has been accepted for publication at the IEEE Conference on Secure and
Trustworthy Machine Learning (SaTML 2026).}\newline Cascading Robustness Verification: Toward Efficient Model‑Agnostic Certification

}

\author{
\begin{tabular}{cc}
\begin{tabular}{c}
\textbf{Mohammadreza Maleki} \\
\textit{Electrical, Computer, and Biomedical Engineering} \\
\textit{Toronto Metropolitan University} \\
Toronto, ON, Canada \\
ORCID: 0000-0003-2871-8797
\end{tabular}
&
\begin{tabular}{c}
\textbf{Rushendra Sidibomma} \\
\textit{Computer Science \& Engineering} \\
\textit{University of Minnesota Twin-Cities} \\
Minneapolis, MN, USA \\
ORCID: 0009-0009-5169-7569
\end{tabular}
\\[4em]
\begin{tabular}{c}
\textbf{Arman Adibi} \\
\textit{Computer and Cyber Sciences} \\
\textit{Augusta University} \\
Augusta, GA, USA \\
ORCID: 0000-0002-0417-8506
\end{tabular}
&
\begin{tabular}{c}
\textbf{Reza Samavi} \\
\textit{Electrical, Computer, and Biomedical Engineering} \\
\textit{Toronto Metropolitan University, Vector Institute} \\
Toronto, ON, Canada \\
ORCID: 0000-0001-6768-0168
\end{tabular}
\end{tabular}
}

 \maketitle

\begin{abstract}
Certifying the robustness of neural networks (NNs) against adversarial examples remains a major challenge in trustworthy machine learning. Providing formal guarantees that inputs remain robust against all adversarial attacks within a perturbation budget often requires solving non-convex optimization problems. Hence, incomplete verifiers, such as those based on linear programming (LP) or semidefinite programming (SDP), are widely used because they scale efficiently and substantially reduce the cost of robustness verification compared to complete methods. We identify the limitations of relying on a single incomplete verifier, which can underestimate robustness due to \textit{false negatives} arising from loose approximations or \textit{misalignment} between training and verification methods. In this work, we propose \textit{Cascading Robustness Verification (CRV)}, which goes beyond an engineering improvement by exposing fundamental limitations of existing robustness metric and introducing a framework that enhances both reliability and efficiency in verification. CRV is a model-agnostic verifier, meaning that its robustness guarantees are independent of the model's training process. The key insight behind the CRV framework is that when using multiple verification methods, an input is certifiably robust as long as one method verifies the input as robust. Rather than relying solely on a single verifier with a fixed constraint set, CRV progressively applies multiple verifiers to balance the tightness of the bound and computational cost. Starting with the least expensive method, CRV halts as soon as an input is certified as robust; otherwise, it proceeds to more expensive methods. For each computationally expensive method, we introduce a \textit{Stepwise Relaxation Algorithm (SR)} that incrementally adds constraints and checks for certification at each step, thereby avoiding unnecessary computation. Our theoretical analysis demonstrates that CRV consistently achieves equal or higher verified accuracy across all settings compared to powerful but computationally expensive incomplete verifiers in the cascade, such as SDP-based methods, while significantly reducing verification overhead. Empirical results confirm that CRV certifies at least as many inputs as benchmark approaches, while improving runtime efficiency by up to \(\sim 90\%\).
\end{abstract}

\begin{IEEEkeywords}
Neural Networks, Non-convex Optimization, Robustness, Certified Robustness
\end{IEEEkeywords}

\section{Introduction}
State-of-the-art classifiers can fail catastrophically under imperceptible adversarial perturbations, known as adversarial examples, small input changes that significantly alter a model’s output while remaining invisible to human perception~\cite{pelekis2025adversarial}. This vulnerability limits the deployment of neural networks (NNs) in safety-critical domains such as autonomous driving and healthcare~\cite{madry2017towards},
motivating extensive research into verifying~\cite{fazlyab2020safety,raghunathan2018semidefinite} and improving~\cite{wong2018provable} NNs robustness. Adversarial robustness of NNs is defined as the maximum perturbation a model can withstand while maintaining correct classifications~\cite{NEURIPS2024_f21a76d6}. Empirical defense mechanisms generate adversarial examples (e.g., via PGD~\cite{madry2017towards} or AutoAttack~\cite{croce2020reliable}) and train NNs to classify them correctly. However, these methods are  
often defeated by adaptive attacks, where the attacker tailors the attack specifically to the defense strategy, exploiting its weaknesses~\cite{li2023sok}.
Safety-critical applications~\cite{cao2021invisible} demand formal guarantees of model behavior under bounded perturbations. To provide such guarantees, robustness verification must address the fundamental challenges of non-convexity and high dimensionality in NNs~\cite{ma2024relu}. Existing approaches are broadly categorized as complete or incomplete~\cite{zhang2022general}. Complete methods aim to provide exact robustness guarantees and, given sufficient time, always return a definite ``yes'' or ``no'' answer to the question of whether an input is robust under verification. However, since robustness verification in NNs is generally NP-hard, it is computationally prohibitive to rely on solvers such as Satisfiability Modulo Theories (SMT)~\cite{pulina2012challenging}, which can produce tight certificates.

\begin{figure}[!t]
    \begin{subfigure}{0.24\textwidth}
        \centering
    \includegraphics[scale=0.06]{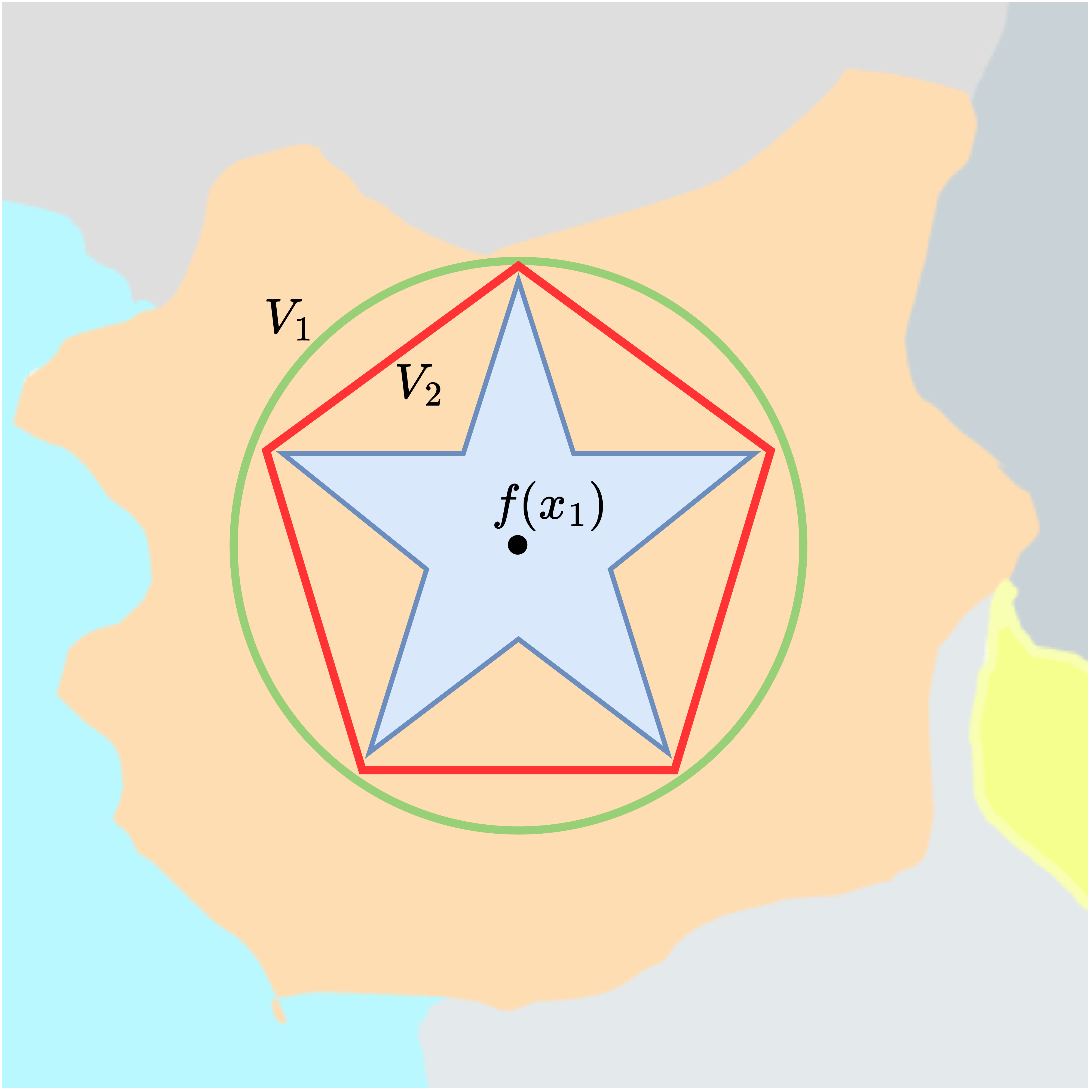} 
        \caption{\(V_1\) \& \(V_2\) agreement} 
        \label{fig:tp}
    \end{subfigure}
    \begin{subfigure}{0.24\textwidth}
        \centering 
    \includegraphics[scale=0.06]{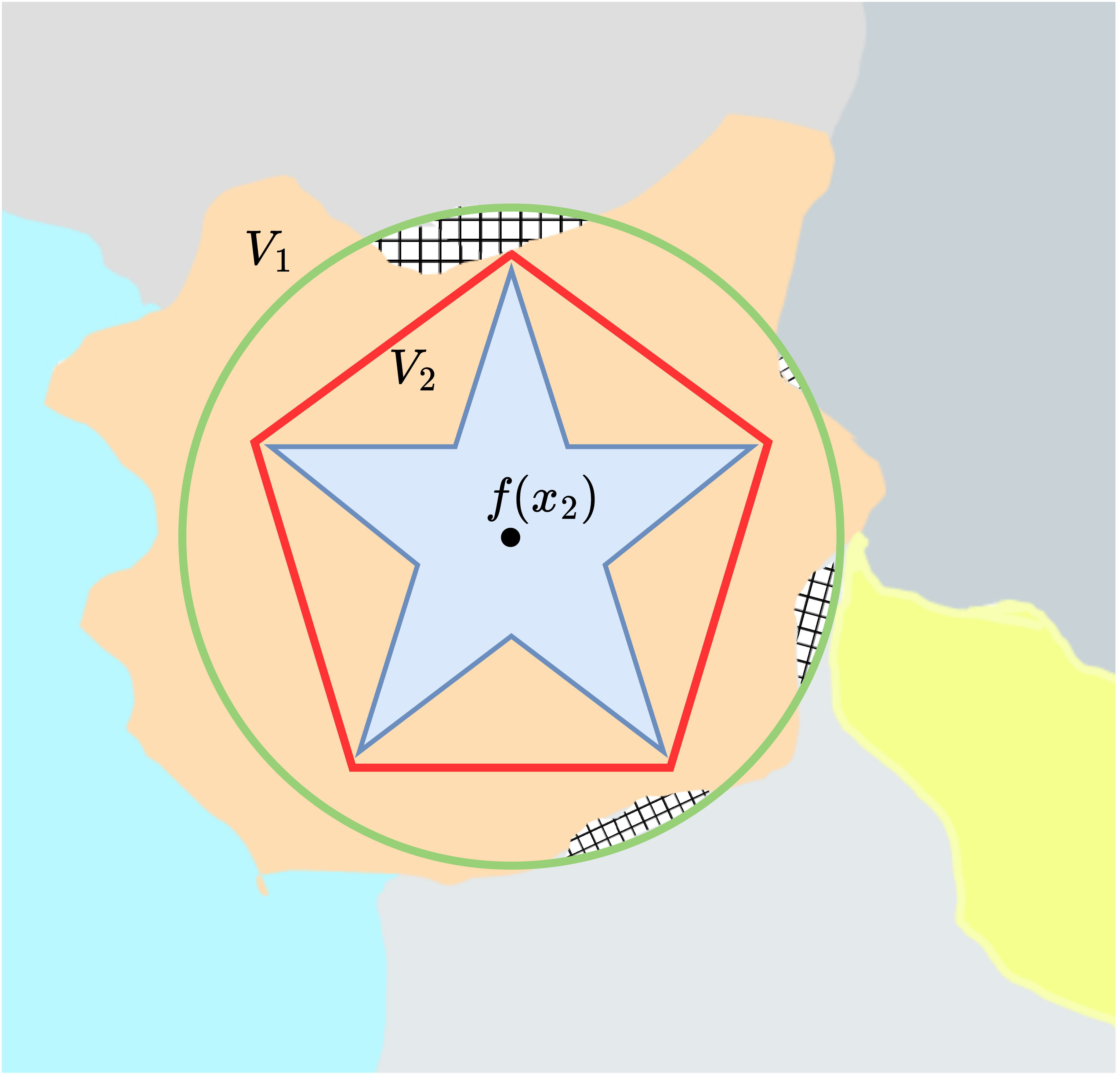} 
        \caption{\(V_1\) \& \(V_2\) disagreement} 
        \label{fig:fn}
    \end{subfigure}
    \caption{Visualization of two individual verifier performance: \( V_1 \) (looser verifier) and \( V_2 \) (tighter verifier). Shaded areas in (a) and (b) shows class boundaries and the star shows the non-convex feasible set. (a) A perturbed input, $x_1$ is correctly certified as robust by both verifiers. (b) A robust input, $x_2$ is verified as non-robust by \( V_1 \) as hatched areas are outside of class boundaries (false negative), but correctly certified by \( V_2 \).}
    \label{fig:def1}
\end{figure}


Incomplete methods, on the other hand, use convex relaxations to enable faster verification, typically at the cost of looser bounds, which may compromise the reliability of robustness certification~\cite{fazlyab2020safety, brix2024fifth}. 
For example, Linear Programming (LP) relaxation methods solve the problem efficiently by replacing nonlinear components with linear constraints at the expense of providing a looser robustness bound~\cite{wong2018provable,wong2018scaling}.  
Semidefinite Programming (SDP) methods reformulate verification as a convex problem over symmetric positive semidefinite matrices with linear constraints, yielding tighter robustness bounds with high computational cost which limits its scalability for deeper architectures~\cite{raghunathan2018semidefinite}. There are also variations of LP and SDP relaxation methods with different levels of complexity and tightness of the bounds (e.g., \cite{dathathri2020enabling,li2023sok,lan2022tight}). 
Existing incomplete verification methods exhibit two key limitations  due to their focus on improving computational efficiency or robustness bound using a \textit{single} verification method. 
First, when a single verifier is used, we may \textit{underestimate robustness} due to false negatives (the inputs that are robust but the verifier identifies them as non-robust) caused by loose approximations. The second limitation is the misalignment between the robust training and verification procedures. This issue, known as \textit{verification-specific behavior}, arises when models trained with one robustness method perform poorly under a misaligned robustness verifier~\cite{li2023sok}.


This paper is motivated by the goal of addressing these two gaps, guided by the intuition that systematically combining multiple verification methods can improve the effectiveness of incomplete robustness verification while incurring substantially lower computational cost and simultaneously mitigating the verification misalignment problem. 
The example in Fig.~\ref{fig:def1} illustrates our motivation and key insights. Consider two robustness verifiers: $V_1$ (e.g., an efficient LP-based method with a looser bound, shown in green) and $V_2$ (e.g., a more expensive SDP-based method with a tighter bound, shown in red). For the inputs in this example, the achievable upper bound on robust accuracy is $48\%$ (see Section~\ref{sec:experiment}), as shown in Fig.~\ref{fig:toy-example}. A subset of inputs ($14\%$), such as $x_1$ in Fig.~\ref{fig:def1}a, can be certified as robust by $V_1$ at low computational cost, whereas the remaining inputs, such as $x_2$ in Fig.~\ref{fig:def1}b, require the tighter but more expensive verifier $V_2$. Intuitively, applying $V_1$ first and invoking $V_2$ only when necessary should attain the same robust accuracy as $V_2$ alone, while incurring high computational cost for only a fraction of inputs. This intuition is empirically confirmed in Fig.~\ref{fig:toy-example}, where the verification time is reduced from $156$ minutes (using $V_2$ alone) to $72$ minutes when $V_1$ and $V_2$ are combined. Notably, this cascading strategy also increases certified robust accuracy from $22\%$ to $36\%$.
Although, in principle, combining $V_1$ and $V_2$ should match the robust accuracy of the tighter verifier $V_2$, in practice the latter often fails to converge for a subset of inputs due to its high computational demands, resulting in unverifiable cases. Moreover, this example illustrates the verifier misalignment issue: the neural network is trained using an SDP-based approach~\cite{raghunathan2018semidefinite}, yet verification is performed using LP-cert~\cite{wong2018provable} and an SDP formulation that differs from the one used during training (SDP-cert). This misalignment between training and verification procedures results in LP-cert and SDP-cert achieving different certified robust accuracies. By integrating both verifiers within our cascading framework, we not only make the certification computationally significantly more efficient, but also exploit their complementary strengths, thereby certifying a strictly larger set of inputs than would be attainable using either verifier in isolation. 
\begin{figure}[!t]
        \centering \includegraphics[width=0.445\textwidth]{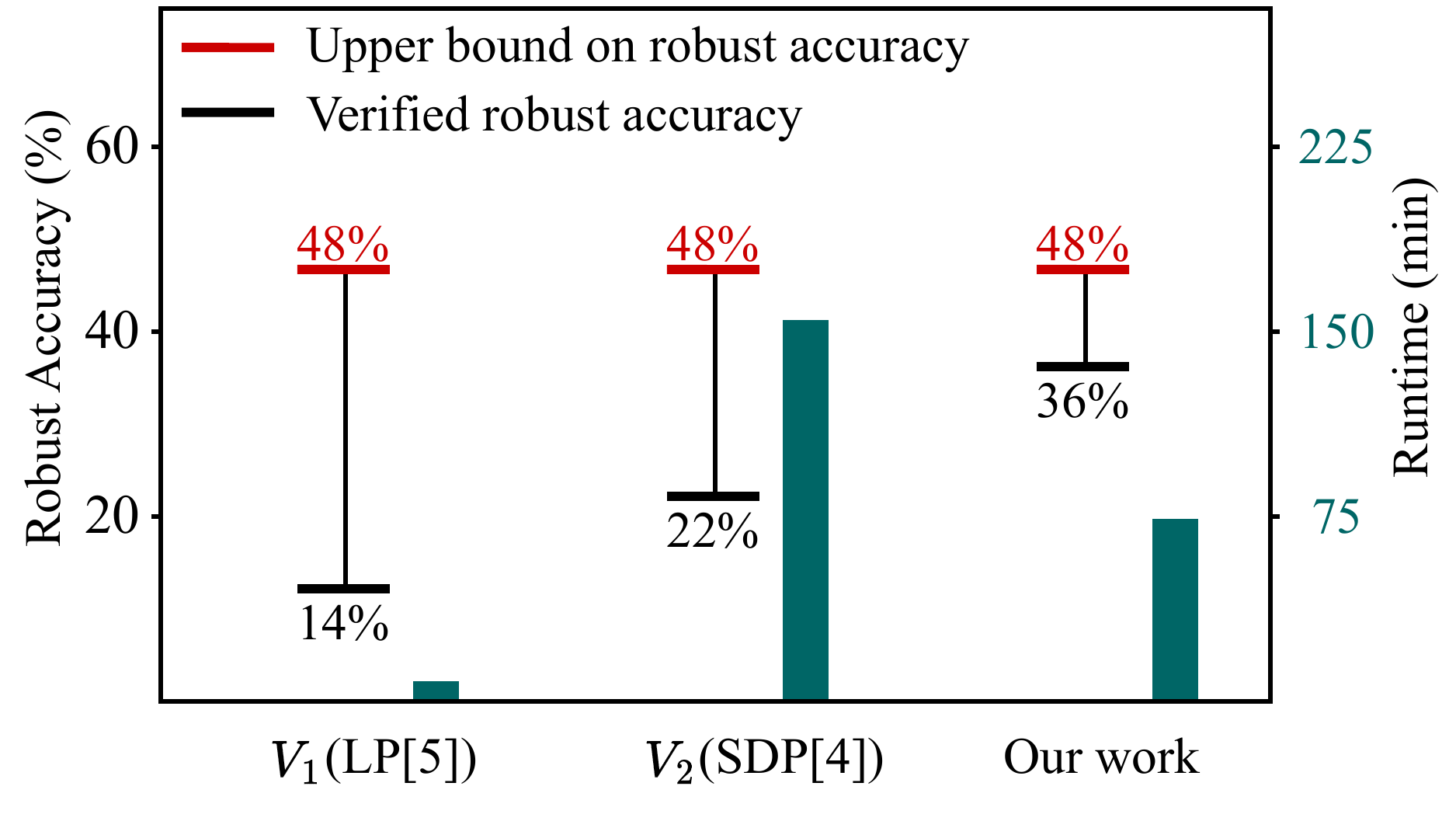} 
    \caption{Robust accuracy and running time of \( V_1 \), \( V_2 \), and our proposed method of combined  \(V_1 \)\&\( V_2 \) to verify MNIST under $\ell_\infty$ adversarial perturbations.}
    \label{fig:toy-example}
\end{figure}

\noindent\textbf{Contributions.} We are making the following contributions: (1) We identify and formalize the limitations of relying on a single incomplete verifier, addressing two problems of {underestimated robustness} and misalignment between the training and verification procedures.  
Our proposed framework is \textit{model-agnostic}, meaning its certification performance is independent of training procedure.   
(2) We introduce \textit{Cascading Robustness Verification (CRV)}, a multi-stage, model-agnostic verifier for certifying the robustness of NNs under norm-bounded adversarial perturbations. CRV strategically combines multiple verification methods in a hybrid pipeline that achieves both \textit{robustness precision} and \textit{computational efficiency}.
(3) We propose \textit{Stepwise Relaxation Algorithm (SR)} to further improve the efficiency of individual verifiers by applying successively tighter relaxations. The procedure begins with a coarse approximation; if robustness is certified, verification halts early. Otherwise, stricter constraints are applied to tighten the bound. We show that by integrating SR into the CRV framework (\textit{CRV-SR}), we can enhance robustness precision while reducing computational cost.
(4) We develop \textit{Fast Stepwise Relaxation Algorithm (FSR)} to further accelerate verification by selectively retaining only the most impactful constraint sets, discarding those with negligible effect on bound quality. Combining FSR with CRV yields \textit{CRV-FSR}, which significantly reduces runtime with only minimal loss in certified robustness.

\section{Related Work}\label{sec:appendix-relatedwork} 
Certified robustness of NNs has attracted extensive research. Approaches can be broadly classified into \emph{complete} and \emph{incomplete} verification methods. Complete methods (e.g., mixed-integer linear programming~\cite{bastani2016measuring} or SMT~\cite{pulina2012challenging} solvers) rely on discrete optimization to guarantee robustness for any given network while producing tight certificates. Despite their tight guarantees, these methods are often computationally expensive, with worst-case complexity growing exponentially with network size~\cite{raghunathan2018semidefinite}. 

In contrast, incomplete methods use convex relaxations and overapproximations to enable faster verification. These techniques transform the non-convex certification problem into a convex one by replacing sources of non-convexity, such as nonlinear activation functions, with convex approximations. Although such relaxations significantly improve computational efficiency, they come at the cost of loose bounds~\cite{fazlyab2020safety,raghunathan2018semidefinite}, which can affect the quality of robustness certification. Recent work focuses on improving these convex relaxations to tighten bounds while preserving scalability~\cite{fazlyab2020safety}. We now review key lines of work in this space, including LP- and SDP-based relaxations, as well as metrics for robustness evaluation.

\subsection{LP–Based Methods}

Early research employed \emph{linear relaxations} and \emph{bound propagation} to estimate the adversarial loss.
Reference~\cite{wong2018provable} introduced a convex outer approximation for training ReLU networks with provable robustness under $\ell_\infty$ perturbations. They define an adversarial polytope $\tilde Z_\varepsilon(x)$ over all possible perturbed activations and then derive a dual LP formulation to compute a tractable upper bound on worst-case loss. In this dual, auxiliary variables $v_i$ propagate bounds backward through each layer, handling ReLU nonlinearity. However, the dual LP still scales quadratically with the number of hidden units. Reference~\cite{wong2018scaling} later proposed using random Cauchy projections to approximate the LP dual, yielding \emph{linear} scaling with network size and input dimension.   

Building on these ideas, bound-propagation frameworks like CROWN~\cite{zhang2022general} further automate the dual relaxation to any feed-forward network, computing layerwise pre-activation bounds and propagating them through linear inequalities. Recent work has refined CROWN’s relaxations: for example, \(\beta\)-CROWN~\cite{wang2021beta} introduces per-neuron split constraints to tighten the linear relaxation. By adaptively splitting the feasible region of selected neurons, \(\beta\)-CROWN improves the tightness of bounds with only a modest increase in computational cost. These bound-propagation methods maintain the speed of linear relaxations while achieving higher verified robust accuracy on benchmarks.  

Other LP-based techniques include interval bound propagation (IBP) and zonotope methods, which use simpler but cruder relaxations for very fast verification (often used in training)~\cite{fazlyab2020safety}. In summary, LP-based and bound-propagation methods form the backbone of many modern incomplete verifiers, as they are fast and broadly applicable. This makes the method practical for larger models, though balancing bound tightness and scalability remains a challenge for reliable robustness certification.

\subsection{SDP–Based Methods}

Another approach to improve upper bounds on adversarial loss is to use \emph{SDP} relaxations. Reference~\cite{raghunathan2018semidefinite} proposed one of the first SDP-based bounds for NNs. Their method formulates the worst-case adversarial loss as a quadratic program over the ReLU constraints, which they relax to a convex SDP. This yields significantly tighter bounds than earlier norm-based methods, especially for shallow networks, at the expense of much higher computational cost. They also extended their SDP relaxation to multiple layers by introducing quadratic constraints (QCs) that capture ReLU relationships, effectively relaxing a quadratically constrained quadratic program (QCQP) to an SDP.  

Reference~\cite{fazlyab2020safety} developed a related robust-control approach: they use SDP with various QCs to bound perturbations in ReLU networks, allowing a trade-off between tightness and cost by adding more constraints. To scale SDP methods further,~\cite{dathathri2020enabling} proposed a first-order dual solver for~\cite{raghunathan2018semidefinite}, reducing memory and computational cost. Reference~\cite{lan2022tight} added another layer of efficiency by combining SDPs with reformulation-linearization technique (RLT) cuts: they generate linear inequalities from SDP relaxation constraints on the fly, focusing on the most violated neuron bounds. This layer-wise process can certify much larger networks than naive SDPs.  

Recently, hybrid methods have merged SDP ideas with bound propagation. For example, SDP-CROWN~\cite{chiu2025sdp} integrates semidefinite constraints into the CROWN framework, inheriting both the efficiency of bound propagation and the tightness of SDP-based relaxations. SDP-CROWN adds second-order constraints (from SDP relaxations) into the linear propagation steps, yielding much tighter bounds than standard CROWN while still scaling to deep networks. In general, SDP-based verifiers produce the strongest certificates among convex relaxations, but they require sophisticated approximations (or selective constraint inclusion) to run on large models.

\subsection{Robustness Verification Metrics}

Researchers commonly assess verification methods using multiple metrics, as outlined below, with particular emphasis on bound quality and robust accuracy.


\noindent\textbf{Bound Tightness.} One approach (exemplified by~\cite{fazlyab2020safety}) is to compare the worst-case loss bounds from different relaxations on a fixed input, NN model, and perturbation level \(\epsilon\). However, it only indicates the relative tightness of bounds and lacks robustness verification for well-known datasets such as MNIST. Consequently, this method is seldom used in literature to assess verification quality or to compare it with alternative approaches~\cite{fazlyab2020safety}.  

\noindent\textbf{Robust Accuracy.} More commonly, papers report the fraction of randomly selected inputs that the method certifies as robust for a particular model under a specified perturbation budget, also known as \emph{verified robust accuracy}~\cite{raghunathan2018semidefinite}.
That is, we count how many examples a given verifier can certify as having no adversarial counterexample in the $\varepsilon$-ball. This metric is intuitive and directly comparable across methods on datasets like MNIST or CIFAR. However, it has two major caveats. First, a verifier with a loose relaxation can produce \emph{false negatives}, undercounting robust inputs (a robust input might be certified as non-robust by a verifier). Thus, robust accuracy depends both on the model’s true robustness and the verifier’s tightness. Second, and more subtly, there is often a misalignment between the robustness training procedure and the verifier used for evaluation~\cite{li2023sok}. If a model is trained against one type of bound computation (e.g., IBP) but evaluated with a different verifier (e.g., CROWN), the robust accuracy can drop simply due to this misalignment. Our proposed framework, CRV, mitigates false negatives caused by loose approximations or misalignment between training and verification methods.

\section{Robustness analysis of neural networks}\label{sec:robustness-analysis-of-NN}
\noindent
Consider a neural network \( f: \mathcal{X} \rightarrow \mathcal{Y} \) with one hidden layer consisting of \( m \) neurons. The input space is \( \mathcal{X} \subseteq \mathbb{R}^d \), where \( d \) is the input dimension, and the output space is \( \mathcal{Y} \subseteq \mathbb{R}^L \), representing scores (or logits) across \( L \) classes. The network computes the output scores as \( f(x) \triangleq W_2\,\mathrm{ReLU}(W_1 x) \), where \( W_1 \in \mathbb{R}^{m \times d} \) and \( W_2 \in \mathbb{R}^{L \times m} \) are the weight matrices connecting the input to the hidden layer and the hidden layer to the output layer, respectively.
 For simplicity, we omit bias terms, though they are included in our experiments.

\subsection{Robustness Verification} 
Given an input \( x \in \mathcal{X} \), the network assigns a label \( y \) by selecting the class with the highest score:
\begin{equation}
y \triangleq \underset{1 \leq i \leq L}{\arg\max} f(x)_i,
\end{equation}
where \( f(x)_i\) is the \(i^{th}\) element of \( f(x)\).
To evaluate robustness, we consider the maximum allowable perturbations around \( x \) within an \( \ell_\infty \)-ball of perturbation level \( \varepsilon \) as: 
\begin{equation}\label{eq:perturbation-budget}
B_{\varepsilon}(x) \triangleq \{x' \mid \|x' - x\|_{\infty} \leq \varepsilon\}.
\end{equation}
The worst-case margin \( l^\star(y, y') \) quantifies the robustness of the model at input \(x\) by maximizing the difference between the logit score of any incorrect class \( y' \neq y \), \(f(x')_{y'}\), and the true class \( y \), \(f(x')_y \), over all valid perturbations:
\begin{align}\label{eq:verification-method}
    \begin{split}
l^\star(y, y') \triangleq &\max_{x'} \big(f(x')_{y'} - f(x')_{y}\big)\\
&\text{s.t.} \quad  f(x') = W_2 \operatorname{ReLU}(W_1 x'),\\
&\quad \quad \ x' \in B_{\varepsilon}.
    \end{split}
\end{align}

An input \(x\) is certifiably robust if, for all \( y' \neq y \), the worst-case margin satisfies \( l^\star(y, y') < 0\), ensuring no perturbation within \( B_{\varepsilon}(x) \) leads to misclassification. Conversely, \(x\) is considered non-robust if there exists a \( y' \neq y \) such that \( l^\star(y, y') \geq 0 \), indicating a potential adversarial example within the allowed perturbation level \(\varepsilon\). We need to solve this optimization problem for each input in the dataset to determine the neural network's robustness.
Solving~\eqref{eq:verification-method} is generally intractable due to the non-convexity introduced by ReLU activations~\cite{raghunathan2018semidefinite}. Although there are several complete verification methods, e.g.,~\cite{wang2021beta},
that solve~\eqref{eq:verification-method},  
they are often computationally expensive, with worst-case complexity growing exponentially with the network size~\cite{raghunathan2018semidefinite}. When complete methods are used for robustness verification, the robustness of a neural network over the dataset \(\mathcal{X}\) is defined as \(|S| / |\mathcal{X}|\), where \(S \subseteq \mathcal{X}\) denotes the set of certifiably robust inputs. 

Incomplete verification methods approximate the original non-convex constraints using convex relaxations, e.g., LP or SDP, to compute an upper bound \(L(y, y')\) on the worst-case margin \( l^\star(y, y') \). A certification is obtained when \( L(y, y') < 0 \).
We define the relaxed verification problem in~\eqref{eq:verification-method} as:
\begin{align}\label{eq:verification-method-incomplete}
    \begin{split}
L(y, y') \triangleq &\max_{x'}  \big(f(x')_{y'} - f(x')_{y}\big) \\
&\text{s.t.} \quad  x' \in \mathcal{C}_r \quad \text{for } r = 1, \dots, k,
    \end{split}
\end{align}
where \( \{\mathcal{C}_1, \dots, \mathcal{C}_k\} \) is a convex relaxation of the constraints in~\eqref{eq:verification-method}.

The set of certified robust inputs depends on the tightness of \( L(y, y') \). Loose bounds can lead to false negatives, where robust inputs are verified incorrectly as non-robust. Fig.~\ref{fig:def1} illustrates the class boundaries (in different shades), two verifiers $V_1$ and $V_2$ such that $V_2$ computes tighter bounds than $V_1$, and the non-convex feasible set (denoted by a star). In particular, in Fig.~\ref{fig:def1}(b), the over-approximation produced by $V_1$ (green) crosses the true decision boundary, causing the point $x$ to be declared non-robust (hashed area). In contrast, the tighter over-approximation of $V_2$ (red) correctly certifies $x$ as robust.
\begin{figure*}[t]
    \centering 
        \includegraphics[width=\textwidth]{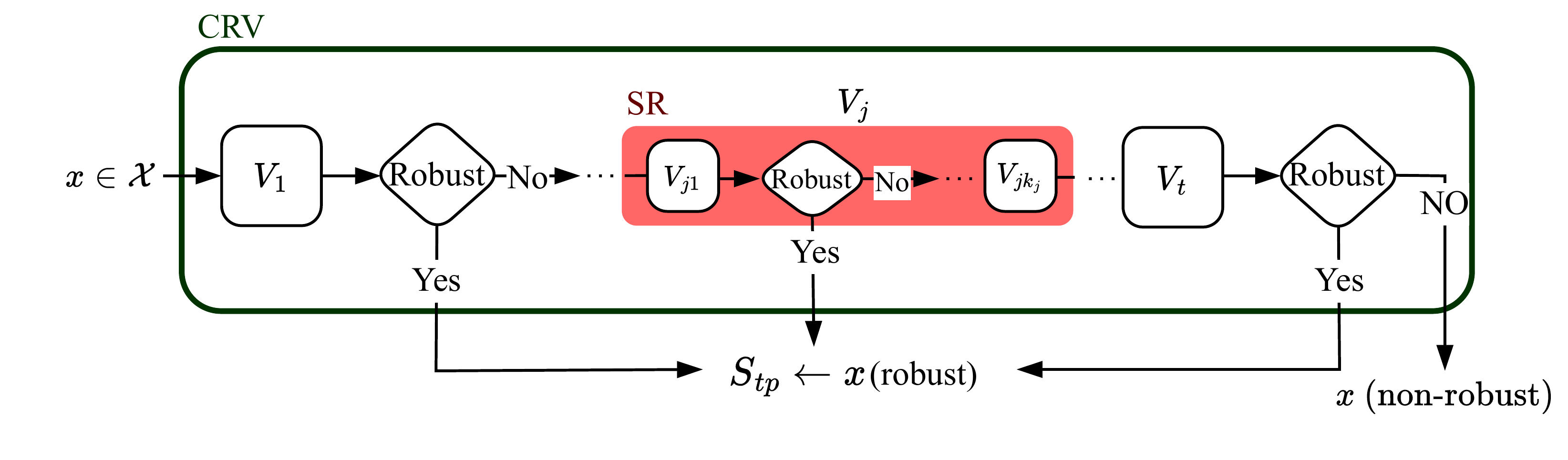}  
    \caption{Diagram of the combined CRV and SR to improve both robustness certification and computational efficiency.}
    \label{fig:diagram-crv}
\end{figure*}
\begin{definition}\label{def:tfn}
We categorize inputs by comparing the worst-case margin \(l^\star(y, y')\) with its estimated upper bound \(L(y, y')\) and define the following sets:
\begin{itemize}
    \item True positives (\(S_{tp}\)) are robust inputs correctly verified as robust:
    \begin{equation*} S_{tp} \triangleq \{x \in \mathcal{X} \mid l^\star(y, y') < 0 \land L(y, y') < 0\}
    \end{equation*}
    \item True negatives (\(S_{tn}\)) are non-robust inputs correctly identified as non-robust:
    \begin{equation*} S_{tn} \triangleq \{x \in \mathcal{X} \mid l^\star(y, y') \geq 0 \land L(y, y') \geq 0\}
    \end{equation*}
    \item False negatives (\(S_{fn}\)) are robust inputs incorrectly classified as non-robust:
    \begin{equation*} S_{fn} \triangleq \{x \in \mathcal{X} \mid l^\star(y, y') < 0 \land L(y, y') \geq 0\}
    \end{equation*}
\end{itemize}
\end{definition}

\begin{remark}
The set of false positives (\(S_{fp}\)),
\begin{equation*}
    \quad \quad S_{fp} \triangleq \{x \in \mathcal{X} \mid l^\star(y, y') \geq 0 \land L(y, y') < 0\}
\end{equation*}
\noindent is empty, as \(L(y, y')\) is an upper bound on \(l^\star(y, y')\); thus, \(L(y, y') < 0\) guarantees \(l^\star(y, y') < 0\).  
\end{remark}

We now define the notions of true robust accuracy and empirical robust accuracy in the context of verification methods.

\begin{definition}\label{def:ra}
\textit{True robust accuracy (TRA)} is the fraction of inputs that are robust to norm-bounded adversarial perturbations:
\begin{equation}\label{eq:RA}  
TRA \triangleq \frac{\lvert S_{tp} \rvert + \lvert S_{fn} \rvert}{\lvert \mathcal{X} \rvert}.
\end{equation}
\end{definition}
In practice, the set \(S_{fn}\) is unobservable, as incomplete verifiers cannot detect all robust inputs. As a result, the empirical robust accuracy, \(\lvert S_{tp} \rvert / \lvert \mathcal{X} \rvert\), serves as a lower bound on TRA. On the other hand, strong empirical attacks such as PGD provide an upper bound on TRA: if the attack succeeds on a fraction \(E_1\) of the inputs \(\mathcal{X}\), then at most \(1 - E_1\) fraction of the inputs \(\mathcal{X}\) can be robust. Together, these bounds form the interval:
\[
    \text{RA} \leq \text{TRA} = 1 - (\lvert S_{tn} \rvert / \lvert \mathcal{X} \rvert) \leq 1 - E_1.
\]
The objective is to tighten the interval for TRA by improving the quality of robustness verification for each input, thereby obtaining the most accurate possible estimate of TRA. From this point on, we refer to empirical robust accuracy as robust accuracy for convenience.


\subsection{Incomplete Verification: Limitations of Robust Accuracy}

\noindent\textbf{False Negatives.} Incomplete verifiers are efficient but often yield loose bounds, which can lead to false negatives~\cite{fazlyab2020safety,li2023sok}, as shown in Fig.~\ref{fig:def1}b. A bound \(L \geq 0\) does not necessarily imply that the worst-case margin \(l^\star(x, y) \geq 0\); the failure may stem from either non-robustness or a loose relaxation. Consequently, the verifier may ``fail to certify'' inputs that are truly robust, causing robust accuracy to underestimate the robustness of the model under the given perturbation~\cite{raghunathan2018semidefinite,dathathri2020enabling}. Since empirical robust accuracy is only a lower bound on TRA (Definition~\ref{def:ra}), tightening relaxations, at additional computational cost, reduces false negatives by certifying such inputs, yielding more accurate and reliable estimates of the model’s TRA.


\noindent\textbf{Training and Verification Misalignment.}  
A second limitation is that robust accuracy is highly dependent on the choice of training and verification methods, which introduces bias in comparisons across approaches. Robust training algorithms typically optimize a specific surrogate objective (e.g., adversarial loss). When such a model is evaluated with a different incomplete verifier, the reported robust accuracy can vary significantly (due to verification-training misalignment). For example, convex relaxations are loose on networks that are not trained against the corresponding relaxation, so a model trained with one method may not verify well under a different relaxation~\cite{raghunathan2018semidefinite}. Empirical studies confirm this effect: models adversarially trained with PGD~\cite{madry2017towards} often achieve strong robustness when evaluated with complete verifiers, but common incomplete verifiers fail to certify a portion of the same inputs. Our experimental results further confirm that robust models trained with a specific relaxation can exhibit lower verified robustness when evaluated using a different relaxation, illustrating the effect of training and verification misalignment.

Taken together, these limitations indicate that relying on a single incomplete verifier provides an unreliable estimate of TRA for robustly trained models.
The next section presents a theoretical investigation of these limitations and how our proposed cascading verification method addresses them.  

\begin{figure*}[t]  
    \centering
    \begin{minipage}{\textwidth}  
        \begin{algorithm}[H]
            \caption{CRV: Cascading Robustness Verification}
            \label{alg:CRV}
            \begin{algorithmic}[1]
            \Statex \textbf{Inputs:} input set \( \mathcal{X} \); model parameters \(W\); verifiers \( V_1, \dots, V_t \) (each with submethods)
            \Statex \textbf{Output:} robust accuracy (\texttt{RA}) over \( \mathcal{X} \)
                \State Initialize \( S_{tp} \gets \emptyset \)
                \For{each \( x \in \mathcal{X} \)} \Comment{Inputs to verify}
                    \For{each \( V_j \in \{V_1, \dots, V_t\} \)} 
                        \State \( S_{tp} \gets \texttt{RobustnessCheck}(x, W, V_j, S_{tp}) \) \Comment{Call Alg.~\ref{alg:SR} to check the robustness of \(x\)}
                        \If{ \( x \in S_{tp} \)} \Comment{Stop if \(x\) verified as robust}
                            \State \textbf{break}
                        \EndIf
                    \EndFor \Comment{Proceed to next input}
                \EndFor
                \State \( \texttt{RA} \gets \frac{\lvert S_{tp} \rvert}{\lvert \mathcal{X} \rvert} \)
                \State \Return \texttt{RA}
            \end{algorithmic}
        \end{algorithm}
    \end{minipage}
\end{figure*}
\section{Cascading robustness verification}\label{sec:crv}
CRV framework follows a \textit{pay as you go} strategy to progressively apply more expensive verification methods with tighter bounds and higher computational complexity as needed. Fig.~\ref{fig:diagram-crv} depicts the operational flow of the framework and its individual components. For every input \( x \in \mathcal{X} \), CRV begins with the least expensive verifier ($V_1$) and only proceeds to more computationally demanding methods if the preceding methods were unable to verify the robustness of the input.

Algorithm~\ref{alg:CRV} formally describes the procedure using a sequence of verification methods \( V_1, \dots, V_t \), ordered by increasing computational complexity.
In Lines 3-8, the robustness of each input \( x \) is determined based on the verifier selected from the set of $V$. For each selected verifier, this algorithm calls a subroutine described in Algorithm~\ref{alg:SR} and applies the verifiers sequentially until one of them, say \( V_j \), certifies \( x \) as robust. In this case, \( x \) is added to the set of true positives \( S_{tp} \) (Line 4), and further verification is skipped. If all verifiers certify $x$ as non-robust, then $x$ is classified as non-robust. Once all inputs in the dataset \( \mathcal{X} \) have been processed, the overall robust accuracy is computed as \( RA = |S_{tp}| / |\mathcal{X}| \). 

\begin{figure}[t]
    \centering 
    \begin{subfigure}{0.24\textwidth}
        \centering 
        \includegraphics[scale=0.065]{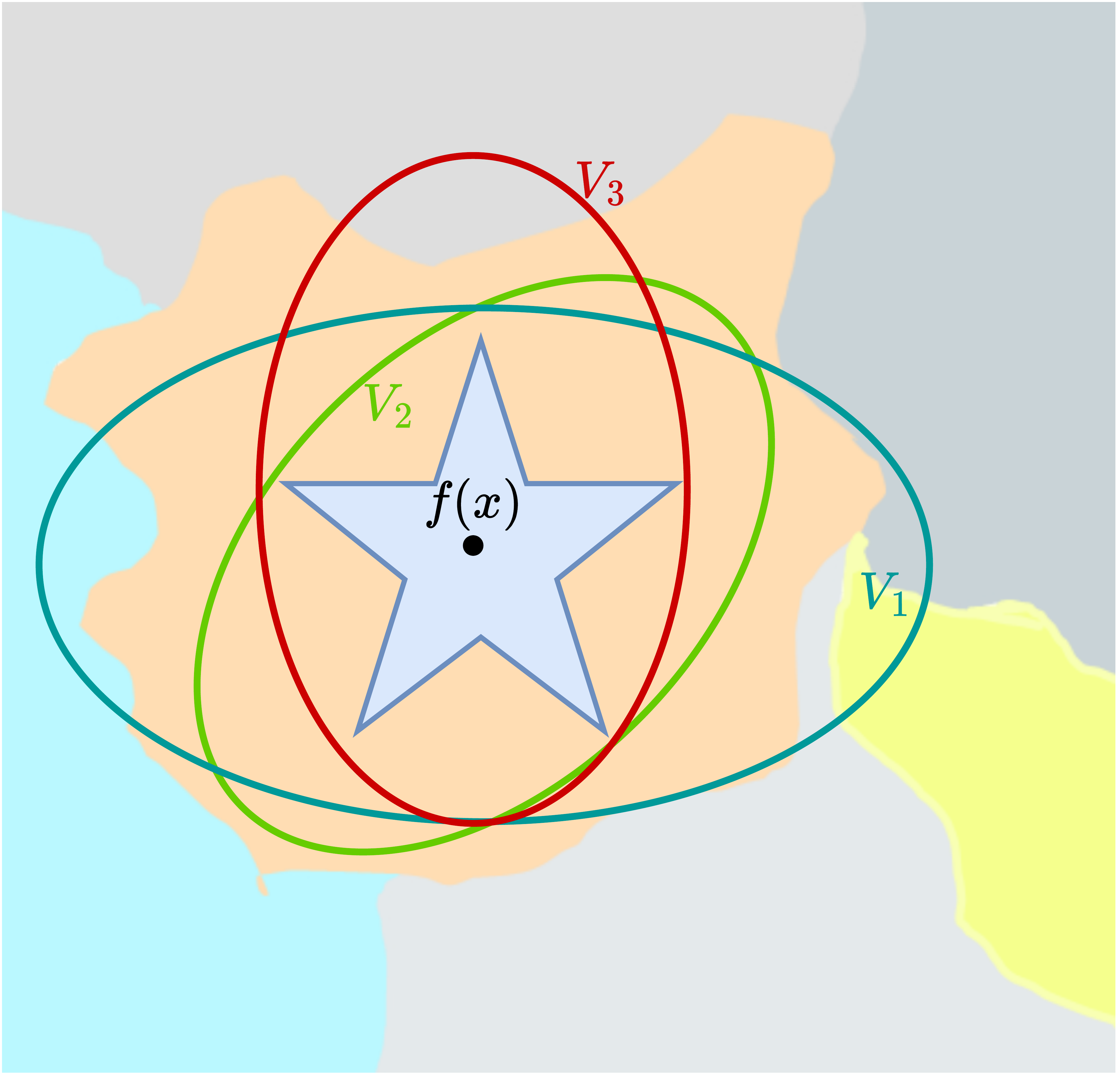} 
        \caption{CRV} 
        \label{fig:crv}
    \end{subfigure} 
    \begin{subfigure}{0.24\textwidth}
        \centering 
        \includegraphics[scale=0.065]{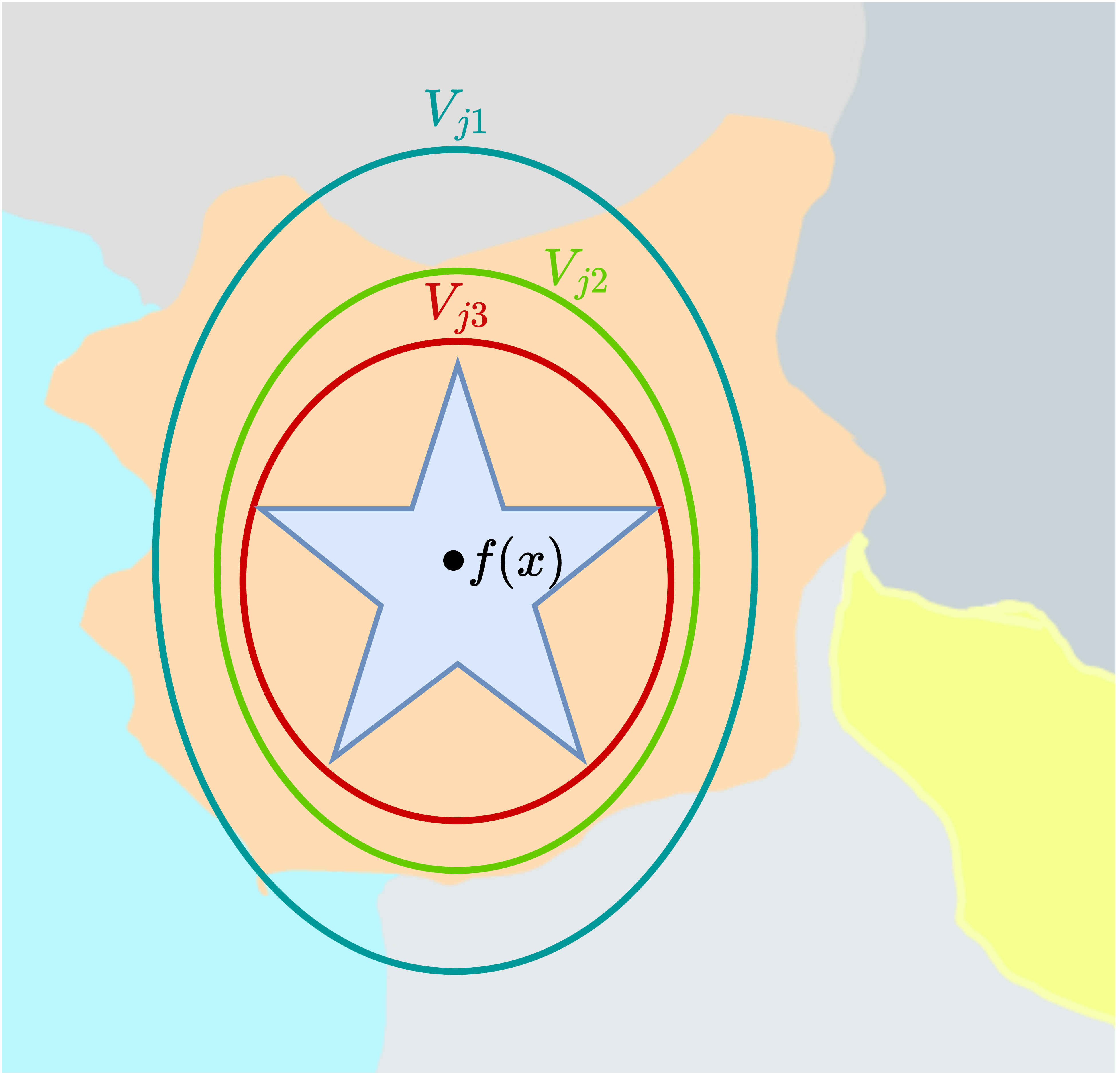} 
        \caption{SR} 
        \label{fig:SR}
    \end{subfigure}  
    \caption{Visualization of classification outcomes under incomplete verifiers. The figures show class boundaries (shaded) and a non-convex feasible set (star). In (a), CRV certifies input when \emph{any} verifier succeeds, which is \(V_2\). 
    In (b), SR iteratively adds constraints to tighten the approximation; here, \(V_{j2}\) certifies the input as robust, so the tighter and more expensive verifier \(V_{j3}\) does not need to be applied.}
    \label{fig:def}
\end{figure}

The intuition behind CRV's computational efficiency and improved robustness measurement is illustrated in Fig.~\ref{fig:crv}. Consider three verification methods \(V_1\), \(V_2\), and \(V_3\), ordered by increasing computational complexity, each providing an upper bound on the non-convex objective in~\eqref{eq:verification-method}. If \(V_2\), which is a less expensive verifier than \(V_3\), certifies robustness without intersecting decision boundaries, the input can be verified without invoking \(V_3\). Moreover, \(V_2\) may correct false negatives produced by both \(V_1\) and \(V_3\). CRV exploits such cases to reduce verification cost while improving the robust accuracy. We provide a detailed explanation of method ordering in Remark~\ref{remark:limitation}.

\begin{remark}[\textbf{On the Ordering of Verification Methods}]\label{remark:limitation}
All verifiers in CRV target the same non-convex robustness objective in~\eqref{eq:verification-method} but differ in how they relax sources of non-convexity, particularly from ReLU activations. The ordering is based on the verifier’s relaxation strength and computational cost: linear relaxations (e.g., LP-based) are applied early, while more complex methods involving quadratic or semidefinite constraints follow. As a result, the ordering becomes straightforward when considering the types of relaxations used. Although runtime may vary with the input, we adopt a fixed ordering based on theoretical and empirical complexity rather than input-specific heuristics. The computational complexity of each method is detailed in~\cite{li2023sok}. When the relative computational complexity of two verifiers are not accessible, a brief empirical evaluation on a small representative subset can be used to fix a consistent ordering. If computational complexities are comparable, prioritizing the tighter verifier is preferable, as certifying robust inputs early avoids redundant execution of later verification stages. CRV balances certification strength and efficiency while remaining modular and easily extensible to future incomplete verifiers or architectural variations.
\end{remark}

The CRV framework \textit{avoids unnecessary computation} by skipping expensive verifiers for inputs already certified by earlier, less expensive methods. Since the more expensive method ($V_t$) will eventually be used when faced with a challenging input (an input that none of the preceding methods identified as robust), CRV \textit{achieves robust accuracy at least as high as the best individual method}. Importantly, using tighter approximations later in the cascade also helps correct false negatives produced by earlier, looser verifiers. This behavior improves robust accuracy, thereby yielding more reliable TRA estimates. In addition, by aggregating multiple verifiers, CRV also reduces overreliance on any single verification method and \textit{addresses potential misalignment} between training-time robustness objectives and verifiers. For example, a model trained using an SDP-based verifier may appear less robust when evaluated with an LP-based method. CRV mitigates such underestimation by allowing all available verification methods to contribute, yielding a \textit{more accurate} and \textit{fair estimation of TRA}. 

The robust accuracy achieved by CRV is represented in the following theorem.
\begin{theorem}[Robust Accuracy under CRV]\label{proof-theorem-crv}
Let \( V_1, \dots, V_t \) be verification methods with increasing cost, and let \( RA_i \) and \( S_{tp_i} \) denote the robust accuracy and true positive set of method \( V_i \), respectively, then the overall robust accuracy achieved by CRV satisfies:
\begin{equation}
    RA = \frac{\left|\bigcup\limits_{i=1}^{t} S_{tp_i}\right|}{\left|\mathcal{X}\right|},
\end{equation}
with the following lower bound: 
\begin{equation}  
    RA \geq \max_{1 \leq i \leq t} RA_i.  
\end{equation} 
The total verification cost (TVC) for CRV is given by:
\begin{equation}
    TVC = \left|\mathcal{X}\right| T_1 + \cdots + \left|\mathcal{X} \setminus \bigcup\limits_{i=1}^{t-1} S_{tp_i}\right| T_t,
\end{equation}
where \( T_j \) denotes the average runtime per input for method \( V_j \).
\end{theorem}

\begin{proof}
Let \( V_1, V_2, \dots, V_t \) denote verification methods ordered by increasing computational complexity, as described in~\cite{li2023sok}. Let \( S_{tp_1}, S_{tp_2}, \dots, S_{tp_t} \) denote the sets of true positives (i.e., correctly certified robust inputs) identified by each method, and let \( \mathcal{X} \) denote the full dataset. The corresponding robust accuracies are \( RA_i = \frac{|S_{tp_i}|}{|\mathcal{X}|} \).

Under CRV, verification proceeds sequentially. \(V_1\) is applied first to certify \(S_{tp_1}\). The remaining inputs \( \mathcal{X} \setminus S_{tp_1} \) are passed to \(V_2\), which certifies \(S_{tp_2} \setminus S_{tp_1}\), and so on. The total set of robust inputs certified by CRV is:
\[
S_{\text{CRV}} = S_{tp_1} \cup (S_{tp_2} \setminus S_{tp_1}) \cup (S_{tp_3} \setminus \bigcup_{i=1}^{2} S_{tp_i}) \cup \cdots \cup (S_{tp_t} \setminus \bigcup_{i=1}^{t-1} S_{tp_i}).
\]
Hence, the robust accuracy under CRV is:
\begin{equation}\label{eq:the1-1}
    RA = \frac{|S_{\text{CRV}}|}{|\mathcal{X}|}.
\end{equation}

By properties of set union and disjoint differences, we simplify~\eqref{eq:the1-1} as:
\begin{equation}\label{eq:the1-2}
    RA = \frac{\left|\bigcup\limits_{i=1}^{t} S_{tp_i}\right|}{|\mathcal{X}|}.
\end{equation}

From standard set cardinality bounds, we have:
\begin{equation}\label{eq:the1-3}
    \max_{1 \leq i \leq t} |S_{tp_i}| \leq \left|\bigcup\limits_{i=1}^{t} S_{tp_i}\right| \leq \sum_{i=1}^{t} |S_{tp_i}|.
\end{equation}
Dividing by \( |\mathcal{X}| \), we obtain:
\begin{equation}\label{eq:the1-4}
    \max_{1 \leq i \leq t} RA_i \leq RA \leq \sum_{i=1}^{t} RA_i.
\end{equation}

\noindent\textbf{Verification Cost.} Let \(T_j\) denote the average runtime per input for method \(V_j\). TVC under CRV is the sum of costs incurred at each stage. Initially, all inputs are verified by \(V_1\), costing \( |\mathcal{X}| T_1 \). For each subsequent method \(V_j\), only the inputs not yet verified as robust are passed forward:
\begin{equation}
    TVC = |\mathcal{X}| T_1 + \left|\mathcal{X} \setminus S_{tp_1}\right| T_2 + \cdots + \left|\mathcal{X} \setminus \bigcup_{i=1}^{t-1} S_{tp_i}\right| T_t.
\end{equation}
This concludes the proof.
\end{proof}

The CRV framework described so far is limited when earlier verifiers in the sequence verify many inputs as non-robust, leading to overreliance on more expensive verification methods. To reduce computational cost for such scenarios, we introduce \emph{SR}, a strategy that incrementally tightens bounds within each method, aiming to further reduce false negatives and minimize computational overhead within each method.

\subsection{Stepwise Relaxation Algorithm (SR)}\label{sec:Stepwise-Relaxation}
SR arranges the native relaxation components of a verifier into a fixed sequence of constraint blocks and applies them progressively from the loosest (fastest) to the tightest (slowest).  Concretely, each verifier \( V_j \), consists of a sequence of submethods \( \{V_{j1}, \dots, V_{jk_j}\} \), where each \( V_{jm} \) solves~\eqref{eq:verification-method-incomplete} using the constraint set \(\bigcap\limits_{i = 1}^{m} \mathcal{C}_i\). This ordering is \emph{verifier-specific} and fixed a priori (i.e., it is not data- or input-dependent), similar to how verifiers are ranked in CRV.
 SR  starts from its loosest submethod \(V_{j1}\) and incrementally tightens the relaxation constraints, at increasing computational cost, through submethods \(V_{j2}, \dots, V_{jk_j}\).

SR is based on the principle that adding constraints, which makes the feasible set smaller for an optimization problem, leads to tighter bounds, typically at the expense of increased computational cost, as stated in the following proposition.

\begin{proposition}\label{proposition-1}
Let two maximization problems \(A\) and \(B\) have the same objective with feasible sets \(\mathcal{F}_B \subseteq \mathcal{F}_A\). Then the optimal value \(L_A\) of problem \(A\) is an upper bound on the optimal value \(L_B\) of problem \(B\); that is, \(L_A \geq L_B\).
\end{proposition}

\noindent The proof follows from~\cite[Chapter 5]{boyd2004convex} and is provided in Appendix~\ref{sec:proposition-1}. This proposition directly informs our certification problem, where we aim to find an upper bound \(L(y, y')\) on the optimal value \(l^\star(y, y')\) in~\eqref{eq:verification-method}, leading to the following corollary.

\begin{corollary}\label{corollary-1}
Let verification submethods \(A\) and \(B\) of \(V_j\), with feasible sets satisfying \(\mathcal{F}_{B} \subseteq \mathcal{F}_{A}\), then:
\begin{enumerate}
    \item \(S_{\mathrm{fn}_{B}} \subseteq S_{\mathrm{fn}_{A}}\): Verification submethod \(V_{B}\) produces fewer false negatives.
    \item \(S_{\mathrm{tp}_{A}} \subseteq S_{\mathrm{tp}_{B}}\):  Verification submethod \(V_{B}\) certifies more robust inputs.
\end{enumerate}
\end{corollary}

\noindent The proof for this corollary is provided in Appendix~\ref{sec:corollary-1}.  

\begin{algorithm}[H]
    \caption{$\texttt{RobustnessCheck}(x, W, V_j, S_{tp})$}
    \label{alg:SR}
    \begin{algorithmic}[1]
    \Statex \textbf{Inputs:} input \(x\); model parameters \(W\); \(V_j\) with submethods \(\{V_{j1}, \dots, V_{jk_j}\}\); true positives \(S_{tp}\) 
    \Statex \textbf{Output:} updated true positives \(S_{tp}\)
        \For{each submethod \( V_{jm} \in V_j\)}
            \If{ \( x \) is certified by \( V_{jm} \) }
                \State \( S_{tp} \gets S_{tp} \cup \{x\} \) \Comment{if \(x\) verified as robust}
                \State \Return \(S_{tp}\) \Comment{Return to Alg.~\ref{alg:CRV}}
            \EndIf
        \EndFor
        \State \Return \(S_{tp}\) \Comment{Return to Alg.~\ref{alg:CRV} if \(x\) not verified as robust}
    \end{algorithmic}
\end{algorithm}

Empirically, we demonstrate this by comparing SDP-cert (\(f_{\text{SDP}_2}\)) from~\cite{raghunathan2018semidefinite} and SDP-cert-loose (\(f_{\text{SDP}_1}\)), where the latter includes less constraints and thus yields looser bounds. By Proposition~\ref{proposition-1}, \(f_{\text{SDP}_1} \ge f_{\text{SDP}_2}\), and Corollary~\ref{corollary-1} implies that SDP-cert produces fewer false negatives. Table~\ref{tab:tightness-comparison-epsilon-0.1-method1&2} presents empirical results on the same input and network: because SDP-cert-loose uses a looser relaxation, it can declare some robust inputs non-robust, whereas SDP-cert returns tighter (more negative) bounds and correctly certifies those inputs. For example, for true class 8 versus adversarial class 2, SDP-cert-loose yields a positive bound (non-robust) while SDP-cert yields a negative bound (robust), consistent with Corollary~\ref{corollary-1}. Additional explanations and formulas for the relaxations are given in Appendix~\ref{sec:limitation}.

\begin{table}
    \caption{Comparison of SDP-cert-loose and SDP-cert on a single input (true class: 8). A checkmark (\cmark) indicates the input verified as robust (i.e., negative upper bound (value)).}
  \label{tab:tightness-comparison-epsilon-0.1-method1&2}
  \centering
  \resizebox{\columnwidth}{!}{
  \begin{tabular}{lllll}
    \toprule
     & \multicolumn{2}{c}{SDP-cert-loose} & \multicolumn{2}{c}{SDP-cert}\\
    \cmidrule(r){1-1}\cmidrule(r){2-3}\cmidrule(r){4-5}
    Adv Class & Value     & Robust     & Value     & Robust \\
    \midrule
        0 & -0.4185 & \cmark & -0.5873 & \cmark \\
        1 & -0.1374 & \cmark & -0.4687 & \cmark \\
        \textbf{2} & \textbf{0.0355} & \textbf{\xmark} & \textbf{-0.1793} & \textbf{\cmark} \\
        3 & -0.2148 & \cmark & -0.4362 & \cmark \\
        4 & -0.1795 & \cmark & -0.5633 & \cmark \\
        5 & -0.5238 & \cmark & -0.7848 & \cmark \\
        6 & -0.4598 & \cmark & -0.6963 & \cmark \\
        7 & -0.3191 & \cmark & -0.5726 & \cmark \\
        9 & -0.1940 & \cmark & -0.5892 & \cmark \\
    \bottomrule
  \end{tabular}}
\end{table}

Algorithm~\ref{alg:SR} summarizes this procedure. For each input \(x \in \mathcal{X}\), the process starts with \(V_{j1}\). If the input is verified as robust, it is added to the set of true positives \(S_{tp}\) and verification terminates for that input. Otherwise, the next tighter submethod is applied. This continues until the input is either verified as robust, at which point it is added to \(S_{tp}\), or deemed non-robust by the tightest submethod \(V_{jk_j}\). Since each relaxation is at least as tight as its predecessor, Corollary~\ref{corollary-1} ensures that the resulting robust accuracy is guaranteed to improve or remain unchanged, compared to using a single fixed submethod \(V_{jm}\) across the entire dataset.
Fig.~\ref{fig:SR} illustrates a case where a robust input is initially verified as non-robust by a loose submethod \(V_{j1}\), but identified as robust under a tighter submethod \(V_{j2}\), so further verification with even tighter and more expensive constraints such as \(V_{j3}\) is unnecessary. 

By capitalizing on early exits, once verified as robust by \( V_{ji} \) (\( i < k_j \)), we avoid incurring the full cost of the tightest method \(V_{jk_j}\) while still achieving its level of robust accuracy. 
The robust accuracy achieved by SR is characterized in the following theorem.

\begin{theorem}[Robust Accuracy of SR]\label{proof-theorem-stepwise}
Let \( \{V_{j1}, \dots, V_{jk_j}\} \) be a sequence of increasingly tighter relaxations derived from method \(V_j\), with associated robust accuracy values \( \{RA_{j1}, \dots, RA_{jk_j}\} \) and true positive sets \( \{S_{tp_{j1}}, \dots, S_{tp_{jk_j}}\} \), then the overall robust accuracy of SR is
\[
RA_j \triangleq \frac{|S_{tp_{jk_j}}|}{|\mathcal{X}|}.
\]
\end{theorem}

\begin{proof}
Let \( V_{j1}, V_{j2}, \dots, V_{jk_j} \) denote an ordered sequence of methods derived from the base robustness verification method \( V_j \). Let the corresponding robust accuracies be \( RA_{j1}, RA_{j2}, \dots, RA_{jk_j} \), and the associated sets of true positives be \( S_{tp_{j1}}, S_{tp_{j2}}, \dots, S_{tp_{jk_j}} \).

Under SR, verification proceeds sequentially. \(V_{j1}\) is applied first to certify \(S_{tp_{j1}}\). The remaining inputs \( \mathcal{X} \setminus S_{tp_{j1}} \) are passed to \(V_{j2}\), which certifies \(S_{tp_{j2}} \setminus S_{tp_{j1}}\), and so on. The overall robust accuracy can be expressed as the union of incremental improvements:
\begin{equation*}
    RA_j = \frac{\bigl\lvert S_{tp_{j1}} \cup \bigl(S_{tp_{j2}} \setminus S_{tp_{j1}}\bigr) \cup \cdots \cup \bigl(S_{tp_{jk_j}} \setminus \bigcup_{i=1}^{k_j-1}S_{tp_{ji}}\bigr) \bigr\rvert}{\lvert \mathcal{X} \rvert}.
\end{equation*}

By Corollary~\ref{corollary-1}, the sequence of true positive sets is nested:
\[
S_{tp_{j1}} \subseteq S_{tp_{j2}} \subseteq \cdots \subseteq S_{tp_{jk_j}},
\]
which implies that each union simplifies progressively:
\[
S_{tp_{j1}} \cup (S_{tp_{j2}} \setminus S_{tp_{j1}}) = S_{tp_{j2}}.
\]
Applying this iteratively yields:
\begin{equation*}
    RA_j = \frac{\lvert S_{tp_{jk_j}} \rvert}{\lvert \mathcal{X} \rvert},
\end{equation*}
which is equivalent to \( RA_{jk_j} \).
\end{proof}

To further reduce the runtime of SR, we introduce FSR. The key idea is detailed in the following.

\noindent\textbf{Fast Stepwise Relaxation Algorithm (FSR).}  
To further improve efficiency, we introduce FSR, which augments SR with mechanisms to skip unnecessary computations. The key idea is to avoid applying expensive relaxations or constraints when they are unlikely to improve the bound meaningfully. FSR uses two complementary strategies:

\begin{itemize}
    \item \textbf{Redundant-Constraint Pruning.}  
    During preprocessing, constraints are analyzed for overlaps and redundancies. If one constraint’s feasible region is contained within the others, it can be discarded. Likewise, quadratic relaxations that provide negligible tightening are omitted. These checks reduce optimization complexity without affecting correctness.

    \item \textbf{Early Termination via Thresholding.}  
    Each relaxation step is first tested with a lightweight bound-improvement estimate on a subset of inputs. If the gain is below a fixed threshold (e.g., $5\%$), the step is skipped. This ensures that tighter relaxations are applied only when they promise a meaningful improvement.
\end{itemize}

Both components of FSR are verifier-independent in principle: redundant-constraint pruning can be incorporated when a verifier exposes sufficient structural information about its relaxations, whereas early termination via thresholding is fully empirical and provides a robust default when such structural cues are unavailable. This ensures that FSR remains broadly applicable beyond the specific verifiers considered in this work. By combining these strategies, FSR significantly reduces runtime while maintaining certified robustness. Empirical validation of this approach is presented in Section~\ref{sec:experiment} and Appendix~\ref{sec:appendix-fsr}.

\vspace{0.1cm}
\noindent\textbf{Computational Cost.}
In practice, most verification methods yield a small number of submethods (typically \(k \in \{2, 3, 4\}\)), ordered by increasing computational cost: \(V_{j1}, \dots, V_{jk}\), where \(V_{j1}\) is the fastest and \(V_{jk}\) is the most expensive. Let \(T_{jm}\) denote the average runtime of submethod \(V_{jm}\). SR reduces TVC by certifying inputs earlier using looser submethods. 

Let \(\mathcal{X}\) be the set of all inputs. TVC for a $k$-stage relaxation is computed as:
\begin{equation}
    TVC = \lvert \mathcal{X} \rvert T_{j1} + \lvert \mathcal{X} \setminus S_{tp_{j1}} \rvert T_{j2} + \cdots + \lvert \mathcal{X} \setminus S_{tp_{j(k-1)}} \rvert T_{jk},
\end{equation}
where \( T_{ji} \) is the average runtime of \( V_{ji} \), and \(S_{tp_{ji}}\) is the set of inputs certified by that method.

The baseline cost of applying the tightest submethod \(V_{jk}\) to all inputs is \(\lvert \mathcal{X} \rvert T_{jk}\). The relative speedup is then:
\begin{equation}
    \Delta T = \frac{\lvert \mathcal{X} \rvert T_{jk} - TVC}{\lvert \mathcal{X} \rvert T_{jk}}.
\end{equation}

For illustration, assume a 3-stage process with runtimes \(T_{j1} < T_{j2} < T_{j3}\), and 40\% of inputs verified at each of the first two stages:
\[
    TVC \approx \lvert \mathcal{X} \rvert T_{j1} + 0.6 \lvert \mathcal{X} \rvert T_{j2} + 0.2 \lvert \mathcal{X} \rvert T_{j3}.
\]
\begin{equation}
\begin{aligned}
    \Delta T & = \frac{\lvert \mathcal{X} \rvert T_{j3} - (\lvert \mathcal{X} \rvert T_{j1} + 0.6 \lvert \mathcal{X} \rvert T_{j2} + 0.2 \lvert \mathcal{X} \rvert T_{j3})}{\lvert \mathcal{X} \rvert T_{j3}} \\
    & = \frac{T_{j3} - (T_{j1} + 0.6 T_{j2} + 0.2 T_{j3})}{T_{j3}}.
\end{aligned}
\end{equation}

Assuming \(T_{j3} > T_{j2} > T_{j1}\), with \(\frac{T_{j2}}{T_{j3}} = \frac{1}{2}\) and \(\frac{T_{j1}}{T_{j2}} = \frac{1}{2}\), we get:
\begin{equation*}
    \Delta T = \frac{T_{j3} - 0.75 T_{j3}}{T_{j3}} = 0.25.
\end{equation*}
As estimated above, using SR can lead to a 25\% improvement in verification time (relative to \(T_{j3}\)). Additionally, the speedup can be further improved, reaching up to 50\% or more, by skipping intermediate submethods that yield minimal improvement over the previous method. For instance, if submethod \(V_{ji}\) offers less than \(5\%\) bound improvement over \(V_{j(i-1)}\) on a small sample, we can bypass \(V_{ji}\) and move on to \(V_{j(i+1)}\). Empirically, SR achieves up to a \textbf{32\%} speedup with no drop in robust accuracy, while FSR reaches up to \textbf{53\%} speedup with negligible accuracy loss (Appendix~\ref{sec:appendix-experiemt-eps}).


\section{Experiments}\label{sec:experiment}
We evaluate the CRV framework and its components on robustly trained neural networks, comparing against standard certification baselines. Robust accuracy (RA) and runtime are measured under \( \ell_{\infty} \)-norm attacks with \( \varepsilon \in \{0.1, 0.15, 0.2, 0.25\} \). Due to space limitations, we report detailed results for \( \varepsilon = 0.1\) in the main text, along with summary plots showing trends across varying \( \varepsilon \). Full results are provided in Appendix~\ref{sec:appendix-experiemt-eps}.

\noindent\textbf{Networks.}
We evaluate the robustness of two different NN models trained on the MNIST dataset using distinct robust training procedures:  
\begin{enumerate}
    \item \textbf{Grad-NN.} This is a two-layer fully connected NN with 500 hidden nodes, originally introduced in \cite{raghunathan2018semidefinite}. The network was trained using an SDP-based regularization technique that bounds the gradient of the network to enhance robustness. 
    \item \textbf{LP-NN.} This network has the same architecture as Grad-NN but was trained using the LP-based robust training method proposed in~\cite{wong2018provable}, with weights provided by the authors.
\end{enumerate}
Both models are independent of the SDP verifier (SDP-cert) used in our evaluation, ensuring that SDP certification remains agnostic to the training procedure.
\begin{table*}[!t]
  \caption{Certification results on MNIST under \(l_{\infty}\)-norm attacks at \(\varepsilon = 0.10\). RA, average runtime per sample, and speedup (relative to SDP-cert) are reported. LP-cert is fast and, since its runtime (per sample) is significantly lower than SDP-based methods and not reported in prior work~\cite{wong2018scaling, dathathri2020enabling, li2023sok}, we omit it from the runtime and speedup analysis.} 
  \label{tab:CRV}
  \centering
  \resizebox{0.95\textwidth}{!}{
  \begin{tabular}{lcccccc}
    \toprule
    & \multicolumn{3}{c}{Grad-NN} & \multicolumn{3}{c}{LP-NN}\\
    \cmidrule(r){1-1} \cmidrule(r){2-4} \cmidrule(r){5-7}
    Verifier & RA & Runtime (min) & Speedup & RA & Runtime (min) & Speedup\\
    \midrule
        SDP-cert~\cite{raghunathan2018semidefinite}  & 88\% & 328.98 & - & 88\% & 328.98 & - \\
        LP-cert~\cite{wong2018provable}  & 42\% & - & - & 82\% & - & - \\
        CRV  & 88\% & 190.8 & 42.00\% & 88\% & 59.22 & 82.00\% \\
        CRV-SR & 88\% & 134.23 & 59.20\% & 88\% & 41.66 & 87.34\% \\
        \textbf{CRV-FSR} & \textbf{88\%} & \textbf{111.34} & \textbf{66.16\%} & \textbf{88\%} & \textbf{34.55} & \textbf{89.5\%} \\
        PGD Success
        & 10\%  & -      & -     & 12\%  & -      & -     \\
        \bottomrule
    \end{tabular}}
\end{table*}

\vspace{0.1cm}
\noindent\textbf{Certification Procedures.}  
Certification is performed by computing an upper bound on the certification problem, following CRV approach outlined in Section~\ref{sec:crv}. We compare three different certification methods:  
\begin{enumerate}
    \item \textbf{CRV.} The proposed CRV method computes a hybrid upper bound by leveraging both SDP and LP relaxations. Instead of applying the same method to all inputs, CRV first uses LP-cert, a computationally efficient but less precise method, to quickly filter out robust samples that can be verified by less tight overapproximations. For inputs that remain non-robust, CRV then applies SDP-cert, a tighter but computationally more expensive method, to re-verify their robustness~\cite{vandenberghe1996semidefinite, lasserre2002semidefinite, lasserre2004sdp, de2006aspects, yurtsever2021scalable}.     
    \item \textbf{LP-cert.} A certification method based on the LP relaxation, which has been widely used in prior scalable certification approaches~\cite{wong2018provable}. 
    \item \textbf{SDP-cert.} The certification method introduced in~\cite{raghunathan2018semidefinite} derives an upper bound using the maximum norm of the gradient of the network predictions.
\end{enumerate}

\noindent\textbf{Optimization Setup.} The SDP-based certification procedure solves convex problems using the YALMIP toolbox~\cite{lofberg2004yalmip}, with MOSEK as the solver. The LP-based certification is implemented in PyTorch, using standard modules and optimization routines from \texttt{torch.nn} and \texttt{torch.optim}. All experiments, including SDP and LP methods, are conducted on the same 6-core CPU. On our server, the average SDP-cert computation time is approximately 35 minutes per adversarial class on the MNIST dataset. In the worst case, robustness verification for a single input requires checking up to nine adversarial classes.

\noindent\textbf{Dataset Choice for Evaluation.}  
As highlighted in prior work such as \cite{yang2020closer}, meaningful perturbation level $\varepsilon$ for datasets varies significantly due to differences in their data geometry. For instance, the train–train separation, defined as the distance between each training example and its nearest neighbor from a different class in the training set, is about $0.212$ for CIFAR-10 and $0.737$ for MNIST. In practice, typical adversarial attack radii are much smaller, approximately $0.03$ for CIFAR-10 and $0.1$ for MNIST, indicating that CIFAR-10 requires smaller perturbations to produce realistic adversarial changes. As a result, when evaluating across a range of $\varepsilon$ values on complex datasets such as CIFAR-10 or ImageNet, the robustness space is narrow, and differences between verification methods or training strategies become harder to detect. This limited variation in robust accuracy can obscure the benefits of our approach. While we do not have any restrictions preventing us from experimenting on more complex datasets, for this study, we focused on experiments that align with the benchmarks used in prior papers, which are well-suited to addressing the specific limitation discussed here.

\begin{table}[t]
  \caption{Certified robustness on MNIST under \(l_{\infty}\) perturbations at \(\varepsilon = 0.1\). RA, average runtime per sample with standard deviation, and speedup (relative to \(V_{13}\)) are reported.}
  \label{tab:SR}
  \centering
  \resizebox{0.48\textwidth}{!}{
  \begin{tabular}{lccc}
    \toprule
    Verifier & RA & Runtime (min) & Speedup\\
    \midrule 
        \(V_{11}\) & 42\% & 103.56 $\pm$ 65.52 & -\\
        \(V_{12}\) & 88\% & 152.41 $\pm$ 34.23 & -\\
        \(V_{13}\)~\cite{raghunathan2018semidefinite} & 88\% & 328.98 $\pm$ 74.71 & -\\
        SR & 88\% & 231.43 $\pm$ 65.10 & 29.65\% \\
        \textbf{FSR} & \textbf{88\%} & \textbf{191.96 $\pm$ 61.48} & \textbf{41.65\%}\\
    \bottomrule
  \end{tabular}}
\end{table}

\noindent\textbf{Model-Agnostic Certification Guarantee.} Verifier–training misalignment can bias certification outcomes. For example, as shown in Table~\ref{tab:CRV}, LP-cert yields a robust accuracy of $42\%$ for Grad-NN (which follows an SDP-based training method) and $88\%$ for LP-NN (trained to favor LP-based verifiers). Furthermore, when both networks are verified using an independent method not involved in training (SDP-cert), their robust accuracy is identical. This highlights the limitation of relying on a single verifier for robustness evaluation. CRV addresses this by combining multiple verifiers, thereby reducing the influence of any individual misalignment and returning the best estimate of TRA with a significant speedup. In this context, \textit{model-agnostic} refers to CRV’s ability to measure robustness independent of the training procedure, achieved by aggregating diverse verification perspectives to ensure fairer and more reliable certification.

\noindent\textbf{CRV.} As shown in Table~\ref{tab:CRV}, CRV achieves the same RA as SDP-cert (88\%) on Grad-NN while reducing runtime by \textbf{42\%}. On LP-NN, CRV improves RA from 82\% (LP-cert) to \textbf{88\%}, and reduces runtime by \textbf{82\%} compared to SDP-cert. These results demonstrate that CRV effectively improves RA and scalability across models with different training objectives.
\begin{table*}[h!]
    \caption{Comparison of \(V_{13}\) and \(V_{12}\) on a single input (true class: 3) at \(\varepsilon = 0.1\). A checkmark (\cmark) indicates the input verified as robust (i.e., negative upper bound).}
  \label{tab:FSR-b-1}
  \centering
  \resizebox{0.9\textwidth}{!}{
    \begin{tabular}{lccccccc}
    \toprule
     & \multicolumn{4}{c}{Optimal Answer} & \multicolumn{3}{c}{Runtime(sec)}\\
    \cmidrule(r){1-1}\cmidrule(r){2-5}\cmidrule(r){6-8}
    Adv. Class & \(V_{13}\) & \(V_{12}\) & Robust & Loss (\%) & \(V_{13}\) & \(V_{12}\) & Impv.(\%)\\
    \midrule
        0 & -0.5671 & -0.5670 & \cmark & 0.018 &  34.84 & 15.68 & 55.00\\
		1 & -0.4481 & -0.4480 & \cmark & 0.022 & 37.47  & 12.79 & 65.86\\
		2 & -0.5333 & -0.5332 & \cmark & 0.019 & 38.25 & 12.83 & 66.46\\
		4 & -0.6568 & -0.6567 & \cmark & 0.015 & 36.39 & 12.88  & 64.60\\
		5 & -0.4302 & -0.4301 & \cmark & 0.023 & 38.06 & 13.54 & 64.42\\
		6 & -0.8906 & -0.8904 & \cmark & 0.022 & 34.93 & 18.62 & 46.70\\
		7 & -0.6178 & -0.6177 & \cmark & 0.016 & 30.91 & 15.39 & 50.20\\
		8 & -0.4319 & -0.4318 & \cmark & 0.023 & 32.16 & 17.23 & 46.40\\
		9 & -0.7056 & -0.7055 & \cmark & 0.014 & 32.05 & 15.65 & 51.17\\
    \bottomrule
  \end{tabular}}
\end{table*}

\noindent\textbf{SR.}
We evaluate SR and FSR on Grad-NN. Table~\ref{tab:SR} presents RA, average runtime per input, and per-SDP runtime. In SR, we begin by using verifier \(V_{11}\), whose constraint set is generated by omitting the quadratic constraint and two linear constraints \(z \geq 0\) and \(z \geq x\) from the original SDP-cert verifier proposed in ~\cite{raghunathan2018semidefinite}. \(V_{12}\) is formulated by adding back the inequality \(z \geq 0\) to the constraint set. Finally, \(V_{13}\) is formulated by restoring all the constraints, yielding the original SDP-cert method, which is the tightest relaxation.
As shown, \(V_{13}\) achieves the highest RA (88\%) but requires a substantial runtime of 328.98 minutes per sample. SR maintains this accuracy while reducing the runtime by \textbf{29.65\%}, and FSR further improves efficiency with a \textbf{41.65\%} speedup. In contrast, the fastest submethod \(V_{11}\) achieves only 42\% RA. These results highlight that selectively relaxing constraints can substantially reduce verification time without compromising certification quality.

\noindent\textbf{FSR.}  
We empirically validate the effectiveness of FSR by comparing different versions of the verifier, specifically $V_{12}$ and $V_{13}$. Results show that FSR typically reduces runtime by $50$--$60\%$ while altering the certified robustness bounds by less than $0.5\%$. This confirms that the proposed strategy achieve substantial speedups with negligible loss in accuracy.  
Table~\ref{tab:FSR-b-1} illustrates this effect on a representative input, showing that $V_{12}$ consistently achieves similar bounds to $V_{13}$ at a fraction of the computational cost. Additional results and detailed comparisons are provided in Appendix~\ref{sec:appendix-fsr}.  

\noindent\textbf{CRV-SR.}
The full CRV framework integrates SR to optimize the trade-off between verification speed and bound quality. On both Grad-NN and LP-NN, CRV-SR matches the RA of SDP-cert while significantly reducing runtime, achieving a \textbf{59.2\%} speedup on Grad-NN and \textbf{87.34\%} on LP-NN, as shown in Table~\ref{tab:CRV}. Incorporating FSR further improves efficiency, with CRV-FSR reaching up to \textbf{66.16\%} speedup on Grad-NN and \textbf{89.5\%} on LP-NN.

\begin{figure*}[t]
    \centering 
    \begin{subfigure}{0.328\textwidth} 
        \centering 
        \includegraphics[width=\textwidth]{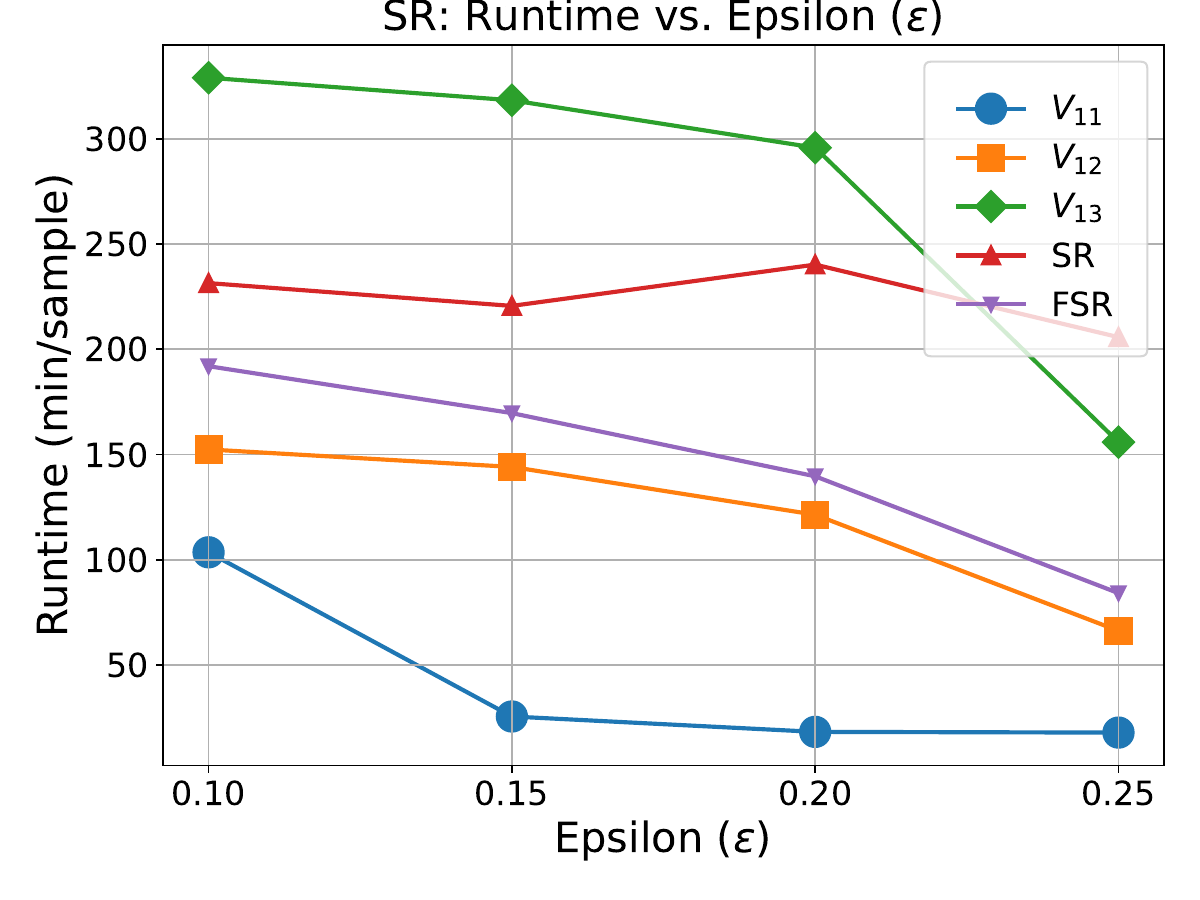} 
        \caption{}
    \end{subfigure} 
    \begin{subfigure}{0.328\textwidth}
        \centering 
        \includegraphics[width=\textwidth]{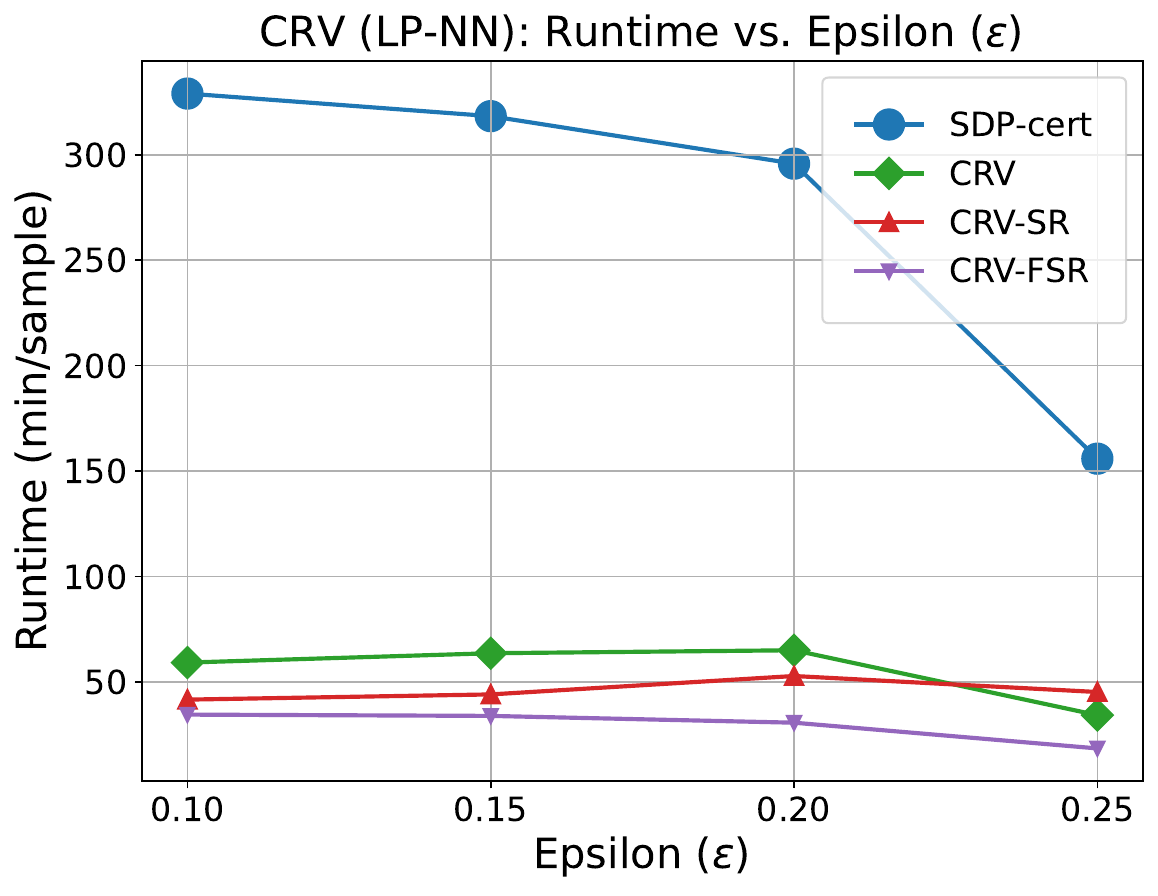} 
        \caption{}
    \end{subfigure} 
    \begin{subfigure}{0.328\textwidth}
        \centering 
        \includegraphics[width=\textwidth]{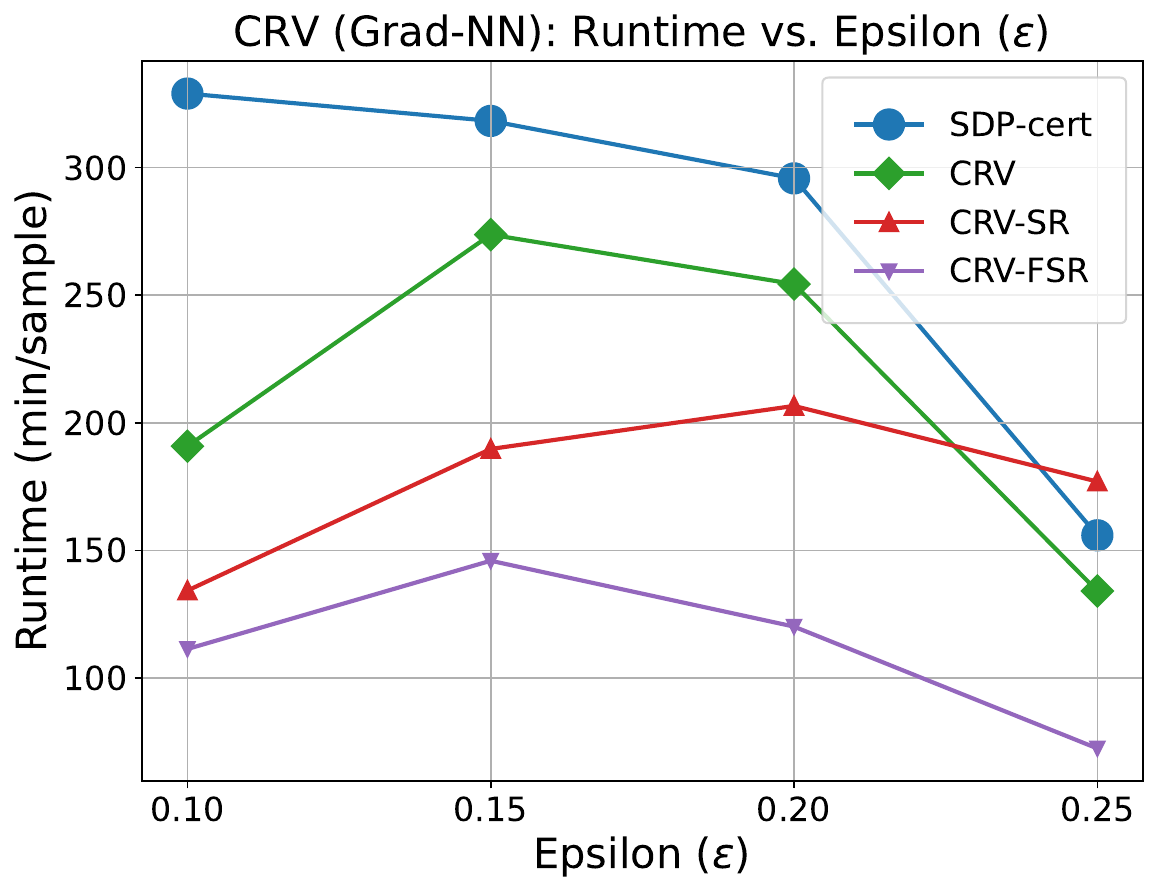} 
        \caption{}
    \end{subfigure} 
    \begin{subfigure}{0.328\textwidth}
        \centering 
        \includegraphics[width=\textwidth]{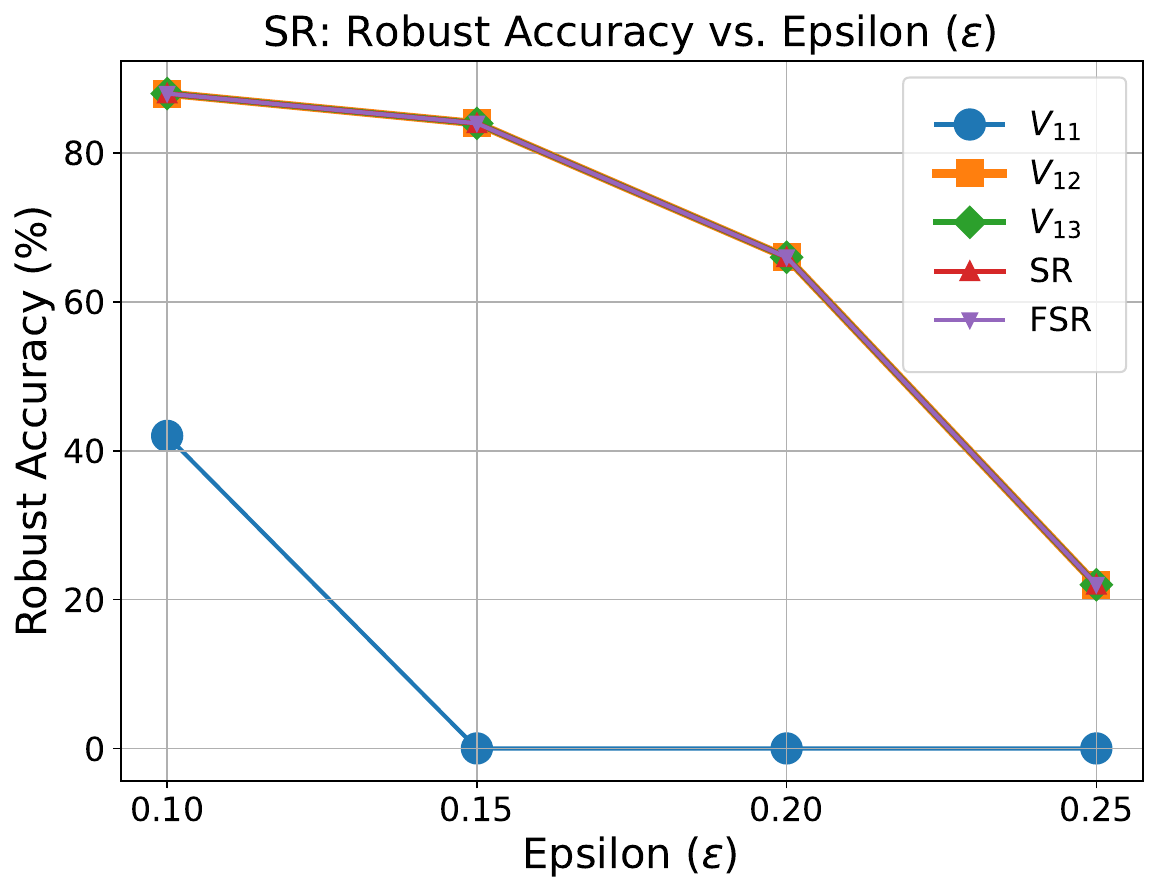}
        \caption{} 
    \end{subfigure} 
    \begin{subfigure}{0.328\textwidth}
        \centering 
        \includegraphics[width=\textwidth]{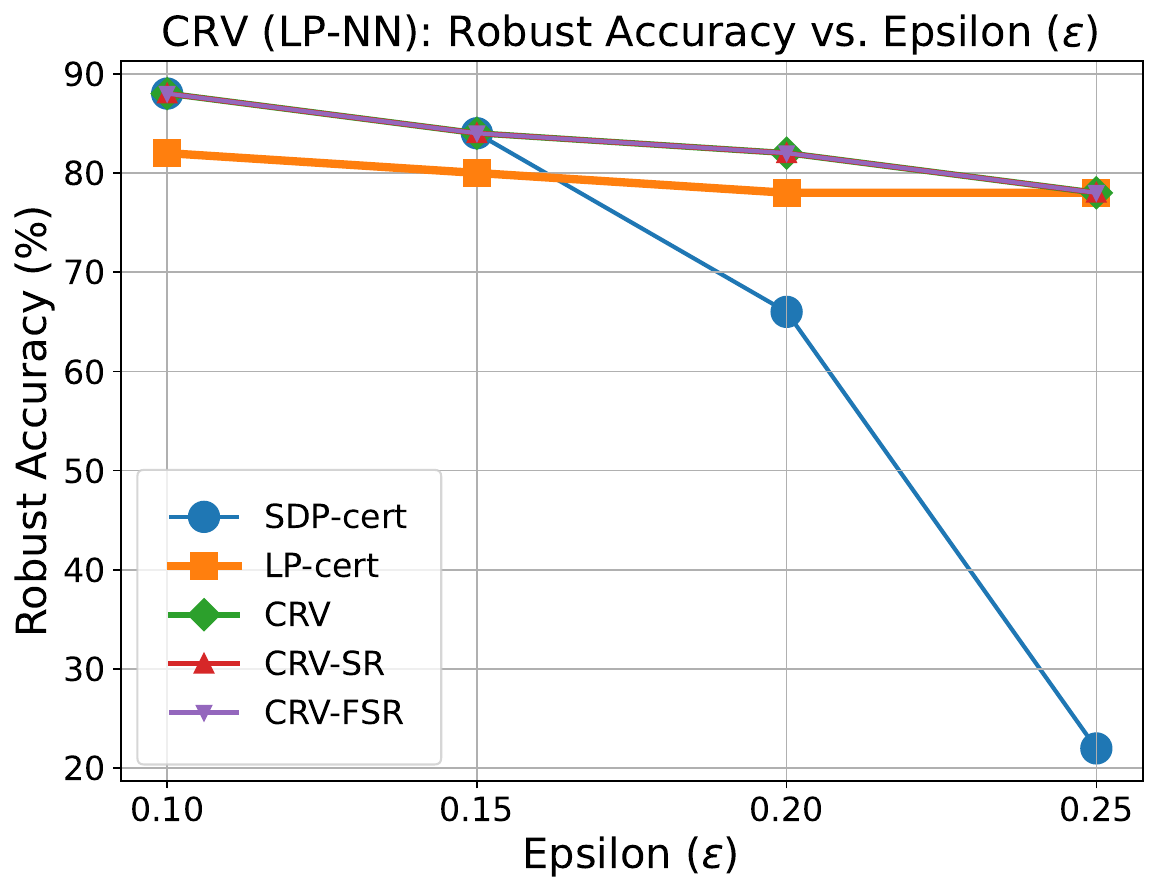} 
        \caption{} 
    \end{subfigure}
    \begin{subfigure}{0.328\textwidth}
        \centering 
        \includegraphics[width=\textwidth]{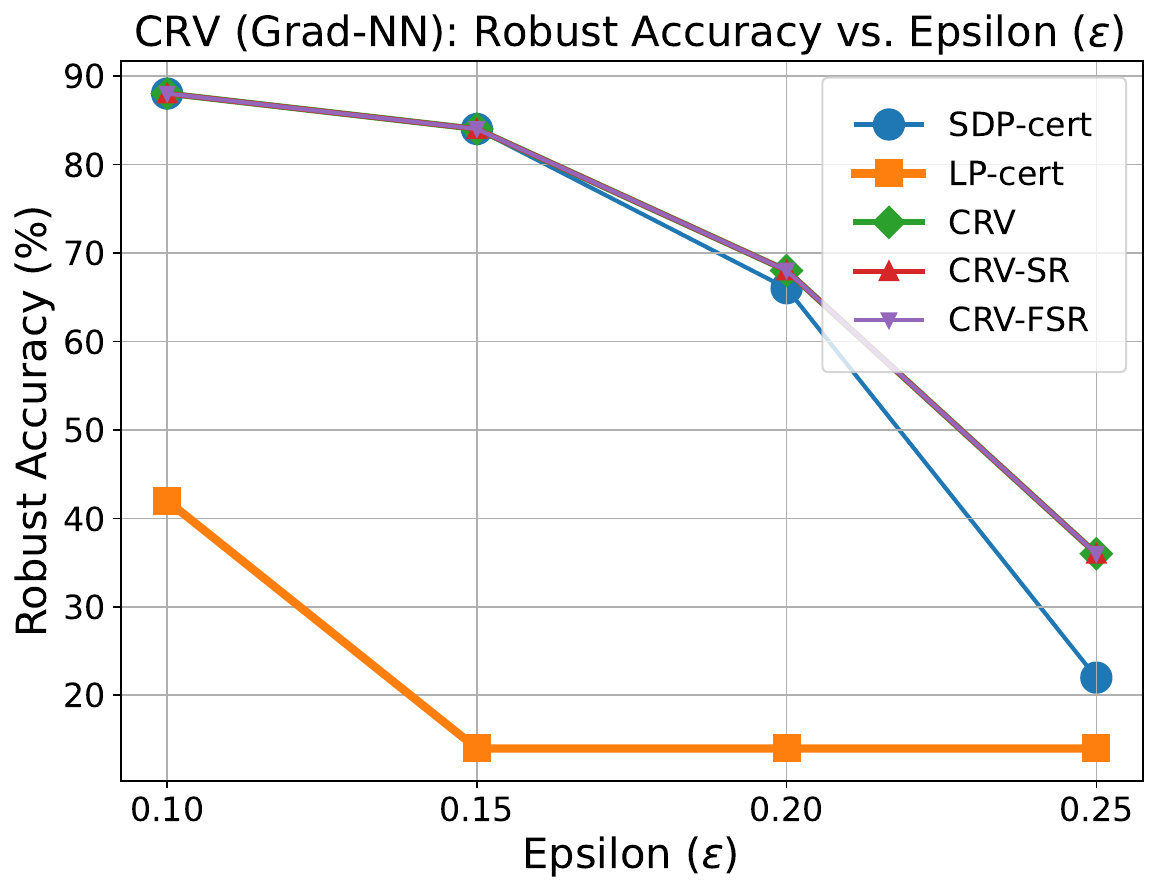} 
        \caption{}
    \end{subfigure}
    \caption{Certified runtime (top row) and robust accuracy (bottom row) across perturbation levels \(\varepsilon \in \{0.1, 0.15, 0.2, 0.25\}\) on Grad-NN and LP-NN. Each column shows: (a, d) SR on Grad-NN, (b, e) LP-NN, and (c, f) Grad-NN under different verification methods. CRV variants achieve robustness comparable to SDP-cert with substantially lower runtime. LP-cert is faster but less reliable on Grad-NN.\\
    \textbf{Note:} In SR, \(V_{13}\) corresponds to the original SDP-cert \cite{raghunathan2018semidefinite} formulation using all constraints. \(V_{12}\) is obtained by removing one linear constraint and the quadratic constraint. \(V_{11}\) further removes one additional constraint from \(V_{12}\), yielding the loosest variant.}
    \label{fig:epsilon}
\end{figure*}
\noindent\textbf{Effect of Perturbation Level \(\varepsilon\).} 
The perturbation budget \(\varepsilon\) specifies the radius of the \(\ell_p\)-ball defining worst-case, task- and dataset-dependent input perturbations; throughout this paper, we focus on certifying robustness for fixed \(\varepsilon\) values under the \(\ell_\infty\) norm.
As the perturbation level increases, robustness certification becomes more difficult: larger \(\varepsilon\) values bring inputs closer to decision boundaries, making loose relaxations more prone to false negatives. Fig.~\ref{fig:epsilon} visualizes this trend. The figure illustrates a general trend as to how SR incrementally applies increasingly tighter submethods, achieving the same robust accuracy as the tightest verifier (e.g., $88\%$ at $\varepsilon=0.1$; see Fig.~\ref{fig:epsilon}d) while reducing runtime (from $328.98$ to $231.43$ min; see Fig.~\ref{fig:epsilon}a), and how FSR further accelerates this process (down to $191.96$ min).

Across all \(\varepsilon \in \{0.1, 0.15, 0.2, 0.25\}\), CRV and its variants consistently maintain high RA while improving runtime. For example, at \(\varepsilon = 0.2\), CRV-FSR achieves \(68\%\) RA on Grad-NN with a runtime of only 120.11 minutes per sample, nearly halving the cost of SDP-cert (295.75 min) while improving the robustness (see Fig.~\ref{fig:epsilon}c,f). On LP-NN, CRV-FSR certifies \(82\%\) of inputs with an \textbf{89.61\%} speedup (see Fig.~\ref{fig:epsilon}b,e). Even at \(\varepsilon = 0.25\), where certification is particularly challenging and most methods fail, CRV-FSR still certifies \(36\%\) of Grad-NN inputs and \(78\%\) of LP-NN inputs with substantial speed gains. 
These results demonstrate that CRV adapts effectively to increasing perturbation levels, thereby maintaining both efficiency and RA. 

As the perturbation level \(\varepsilon\) increases, CRV's average runtime exhibits a non-monotonic behavior (Fig.~\ref{fig:epsilon}c). At intermediate values of \(\varepsilon\) (e.g., \(0.15\)) more inputs fall into a borderline regime where the fast, coarse verifier (LP-cert) fails to certify them, and they must be escalated to tighter, costlier verifiers (e.g., SDP-cert), increasing the mean cost. At larger \(\varepsilon\) (\(0.25\)), a sizable fraction of inputs become clearly non-robust and are rejected early (verification terminates as soon as any adversarial class succeeds), which reduces the number of expensive verifications and thus lowers the average runtime. In short, CRV adaptively concentrates expensive computation where it is most needed, yielding favorable efficiency across perturbation levels.

\noindent\textbf{Empirical Validation of Adversarial Risk.} 
Since computing the exact worst-case adversarial error is intractable, we approximate it using the PGD attack, following the setup in~\cite{raghunathan2018semidefinite}. For $\varepsilon = 0.1$, PGD succeeds on 12\% of LP-NN and 10\% of Grad-NN inputs, both within the CRV upper bound of 12\% potential non-robust instances. The gap between PGD success and SDP-cert’s certification results stems from inputs that are neither certifiably robust nor adversarial under PGD, highlighting the presence of false negatives or borderline cases. The success rates of PGD attacks for $\varepsilon = \{0.15, 0.2, 0.25\}$ are summarized in Table~\ref{tab:CRV-epsilon} (Appendix \ref{sec:appendix-experiemt-eps}).

\noindent\textbf{Verifier Selection.} While CRV can straightforwardly incorporate any post-training verifiers, we intentionally prioritized verifiers and ablations that directly probe verifier–training alignment. The reason is conceptual: most leading verifiers are designed for post hoc certification and do not participate in the training loop that produces alignment bias. Including them in our experiments would therefore primarily confirm the formal guarantee of Theorem~\ref{proof-theorem-crv} (that CRV's certified robust accuracy is at least that of the best included verifier) rather than advance our core goal of studying and mitigating misalignment error. Embedding advanced verifier formulations into the training objective is both promising and orthogonal to this work; thus, it is left for future research. For empirical evaluation, we therefore focus on comparisons and ablations that illuminate CRV’s ability to reduce misalignment bias and tighten the TRA interval, while reporting runtime and robust accuracy trade-offs.

\section{Conclusion}\label{sec:conclusion}
This work highlights key limitations in current robustness measurements for neural networks, particularly the impact of false negatives arising from incomplete methods and the misalignment between training and verification procedures. We proposed CRV, a model-agnostic, multi-step verification framework that integrates multiple verifiers with progressively tighter bounds. CRV-SR improves both the robust accuracy and efficiency of robustness assessment across diverse networks and verification techniques, while maintaining low computational overhead. CRV-FSR further reduces verification time by skipping constraints that yield negligible bound improvements within each verifier.

Our optimized framework for robustness verification builds on existing methods, leaving room for future work to refine its individual components. Targeted improvements to key submodules can further tighten the TRA interval while managing computational complexity. Moreover, the framework can be tailored to specific applications, enabling a balance between robust accuracy and efficiency that reflects the demands of different use cases.

\section*{Acknowledgment}
This research was undertaken thanks in part to funding from the Canada First Research Excellence Fund at Toronto Metropolitan University and Natural Sciences and Engineering Research Council of Canada (NSERC) Discovery grants (\#348100).
\bibliographystyle{IEEEtran}
\bibliography{references}

@article{pelekis2025adversarial,
  title={Adversarial machine learning: a review of methods, tools, and critical industry sectors},
  author={Pelekis, Sotiris and Koutroubas, Thanos and Blika, Afroditi and Berdelis, Anastasis and Karakolis, Evangelos and Ntanos, Christos and Spiliotis, Evangelos and Askounis, Dimitris},
  journal={Artificial Intelligence Review},
  volume={58},
  number={8},
  pages={226},
  year={2025},
  publisher={Springer}
}

@inproceedings{
    madry2017towards,
    title={Towards Deep Learning Models Resistant to Adversarial Attacks},
    author={Aleksander Madry and Aleksandar Makelov and Ludwig Schmidt and Dimitris Tsipras and Adrian Vladu},
    booktitle={International Conference on Learning Representations},
    year={2018},
}

@inproceedings{wong2018provable,
  title={Provable Defenses against Adversarial Examples via the Convex Outer Adversarial Polytope},
  author={Wong, Eric and Kolter, Zico},
  booktitle={International Conference on Machine Learning},
  pages={5286--5295},
  year={2018}
}

@article{raghunathan2018semidefinite,
  title={Semidefinite relaxations for certifying robustness to adversarial examples},
  author={Raghunathan, Aditi and Steinhardt, Jacob and Liang, Percy S},
  journal={Advances in neural information processing systems},
  volume={31},
  year={2018}
}

@inproceedings{NEURIPS2024_f21a76d6,
 author = {Rekavandi, Aref Miri and Farokhi, Farhad and Ohrimenko, Olga and Rubinstein, Benjamin I.P.},
 booktitle = {Advances in Neural Information Processing Systems},
 pages = {134127--134150},
 title = {Certified Adversarial Robustness via Randomized $\alpha$-Smoothing for Regression Models},
 volume = {37},
 year = {2024}
}

@inproceedings{li2023sok,
  title={Sok: Certified robustness for deep neural networks},
  author={Li, Linyi and Xie, Tao and Li, Bo},
  booktitle={2023 IEEE symposium on security and privacy (SP)},
  pages={1289--1310},
  year={2023},
  organization={IEEE}
}

@inproceedings{cao2021invisible,
  title={Invisible for both camera and lidar: Security of multi-sensor fusion based perception in autonomous driving under physical-world attacks},
  author={Cao, Yulong and Wang, Ningfei and Xiao, Chaowei and Yang, Dawei and Fang, Jin and Yang, Ruigang and Chen, Qi Alfred and Liu, Mingyan and Li, Bo},
  booktitle={2021 IEEE Symposium on Security and Privacy (SP)},
  pages={176--194},
  year={2021},
  organization={IEEE}
}

@inproceedings{croce2020reliable,
  title={Reliable evaluation of adversarial robustness with an ensemble of diverse parameter-free attacks},
  author={Croce, Francesco and Hein, Matthias},
  booktitle={International conference on machine learning},
  pages={2206--2216},
  year={2020},
  organization={PMLR}
}

@article{bastani2016measuring,
  title={Measuring neural net robustness with constraints},
  author={Bastani, Osbert and Ioannou, Yani and Lampropoulos, Leonidas and Vytiniotis, Dimitrios and Nori, Aditya and Criminisi, Antonio},
  journal={Advances in neural information processing systems},
  volume={29},
  year={2016}
}

@article{pulina2012challenging,
  title={Challenging SMT solvers to verify neural networks},
  author={Pulina, Luca and Tacchella, Armando},
  journal={Ai Communications},
  volume={25},
  number={2},
  pages={117--135},
  year={2012},
  publisher={IOS Press}
}

@article{fazlyab2020safety,
  title={Safety verification and robustness analysis of neural networks via quadratic constraints and semidefinite programming},
  author={Fazlyab, Mahyar and Morari, Manfred and Pappas, George J},
  journal={IEEE Transactions on Automatic Control},
  year={2020},
  publisher={IEEE}
}

@inproceedings{wong2018scaling,
  title={Scaling provable adversarial defenses},
  author={Wong, Eric and Schmidt, Frank and Metzen, Jan Hendrik and Kolter, J Zico},
  booktitle={Advances in Neural Information Processing Systems},
  pages={8400--8409},
  year={2018}
}

@inproceedings{dathathri2020enabling,
 author = {Dathathri, Sumanth and Dvijotham, Krishnamurthy and Kurakin, Alexey and Raghunathan, Aditi and Uesato, Jonathan and Bunel, Rudy R and Shankar, Shreya and Steinhardt, Jacob and Goodfellow, Ian and Liang, Percy S and Kohli, Pushmeet},
 booktitle = {Advances in Neural Information Processing Systems},
 editor = {H. Larochelle and M. Ranzato and R. Hadsell and M. F. Balcan and H. Lin},
 pages = {5318--5331},
 title = {Enabling certification of verification-agnostic networks via memory-efficient semidefinite programming},
 volume = {33},
 year = {2020}
}

@inproceedings{lan2022tight,
  title={Tight neural network verification via semidefinite relaxations and linear reformulations},
  author={Lan, Jianglin and Zheng, Yang and Lomuscio, Alessio},
  booktitle={Proceedings of the AAAI Conference on Artificial Intelligence},
  volume={36},
  number={7},
  pages={7272--7280},
  year={2022}
}

@article{wang2021beta,
  title={Beta-crown: Efficient bound propagation with per-neuron split constraints for neural network robustness verification},
  author={Wang, Shiqi and Zhang, Huan and Xu, Kaidi and Lin, Xue and Jana, Suman and Hsieh, Cho-Jui and Kolter, J Zico},
  journal={Advances in neural information processing systems},
  volume={34},
  pages={29909--29921},
  year={2021}
}

@book{boyd2004convex,
  title={Convex optimization},
  author={Boyd, Stephen P and Vandenberghe, Lieven},
  year={2004},
  publisher={Cambridge university press}
}

@article{vandenberghe1996semidefinite,
  title={Semidefinite programming},
  author={Vandenberghe, Lieven and Boyd, Stephen},
  journal={SIAM review},
  volume={38},
  number={1},
  pages={49--95},
  year={1996},
  publisher={SIAM}
}

@article{lasserre2002semidefinite,
  title={Semidefinite programming vs. LP relaxations for polynomial programming},
  author={Lasserre, Jean B},
  journal={Mathematics of operations research},
  volume={27},
  number={2},
  pages={347--360},
  year={2002},
  publisher={INFORMS}
}

@article{lasserre2004sdp,
  title={SDP vs. LP relaxations for the moment approach in some performance evaluation problems},
  author={Lasserre, Jean-Bernard and Prieto-Rumeau, Tom{\'a}s},
  year={2004},
  publisher={Taylor \& Francis}
}

@book{de2006aspects,
  title={Aspects of semidefinite programming: interior point algorithms and selected applications},
  author={De Klerk, Etienne},
  volume={65},
  year={2006},
  publisher={Springer Science \& Business Media}
}

@article{yurtsever2021scalable,
  title={Scalable semidefinite programming},
  author={Yurtsever, Alp and Tropp, Joel A and Fercoq, Olivier and Udell, Madeleine and Cevher, Volkan},
  journal={SIAM Journal on Mathematics of Data Science},
  volume={3},
  number={1},
  pages={171--200},
  year={2021},
  publisher={SIAM}
}

@inproceedings{lofberg2004yalmip,
  title={YALMIP: A toolbox for modeling and optimization in MATLAB},
  author={Lofberg, Johan},
  booktitle={2004 IEEE international conference on robotics and automation (IEEE Cat. No. 04CH37508)},
  pages={284--289},
  year={2004},
  organization={IEEE}
}

@article{ma2024relu,
  title={RELU hull approximation},
  author={Ma, Zhongkui and Li, Jiaying and Bai, Guangdong},
  journal={Proceedings of the ACM on Programming Languages},
  volume={8},
  number={POPL},
  pages={2260--2287},
  year={2024},
  publisher={ACM New York, NY, USA}
}

@article{chiu2025sdp,
  title={SDP-CROWN: Efficient Bound Propagation for Neural Network Verification with Tightness of Semidefinite Programming},
  author={Chiu, Hong-Ming and Chen, Hao and Zhang, Huan and Zhang, Richard Y},
  journal={arXiv preprint arXiv:2506.06665},
  year={2025}
}

@article{zhang2022general,
  title={General cutting planes for bound-propagation-based neural network verification},
  author={Zhang, Huan and Wang, Shiqi and Xu, Kaidi and Li, Linyi and Li, Bo and Jana, Suman and Hsieh, Cho-Jui and Kolter, J Zico},
  journal={Advances in neural information processing systems},
  volume={35},
  pages={1656--1670},
  year={2022}
}

@article{brix2024fifth,
  title={The fifth international verification of neural networks competition (vnn-comp 2024): Summary and results},
  author={Brix, Christopher and Bak, Stanley and Johnson, Taylor T and Wu, Haoze},
  journal={arXiv preprint arXiv:2412.19985},
  year={2024}
}

@article{yang2020closer,
  title={A closer look at accuracy vs. robustness},
  author={Yang, Yao-Yuan and Rashtchian, Cyrus and Zhang, Hongyang and Salakhutdinov, Russ R and Chaudhuri, Kamalika},
  journal={Advances in neural information processing systems},
  volume={33},
  pages={8588--8601},
  year={2020}
}
\balance
\newpage
\appendix
\section*{Supplementary Materials}
\subsection{Proof of Proposition~\ref{proposition-1}}\label{sec:proposition-1}
\begin{proof}
Let \( f(x) \) be the objective function for both optimization problems \(A\) and \(B\), with feasible sets \(\mathcal{F}_A\) and \(\mathcal{F}_B\) satisfying \(\mathcal{F}_B \subseteq \mathcal{F}_A\). Let \( x_B^\star \in \mathcal{F}_B \) be the optimal solution  of the problem \(B\), achieving objective value \( L_B = f(x_B^\star) \). Since \( x_B^\star \in \mathcal{F}_A \) and \( L_A \) is the maximum of \( f(x) \) over \(\mathcal{F}_A\), it follows that \( L_A \geq f(x_B^\star) = L_B \). Thus, \( L_A \geq L_B \).
\end{proof}

\subsection{Proof of Corollary~\ref{corollary-1}}\label{sec:corollary-1}
\begin{proof}
By Proposition~\ref{proposition-1}, for every input, the optimal value under method \(A\) is greater than or equal to that under method \(B\); that is, \(L_A \geq L_B\).

First, consider an input that is a false negative under method \(B\), meaning \(L_B \geq 0\) while \(l^\star(y,y') < 0\). Since \(L_A \geq L_B\), it follows that \(L_A \geq 0\) as well, implying that method \(A\) also fails to certify robustness. Therefore, \(S_{\mathrm{fn}_B} \subseteq S_{\mathrm{fn}_A}\).

Next, consider an input verified as robust under method \(A\), meaning \(L_A < 0\) and \(l^\star(y,y') < 0\). Since \(L_A \geq L_B\), it follows that \(L_B < 0\) as well. Thus, any input certified as robust by method \(A\) is also certified as robust by method \(B\), implying \(S_{\mathrm{tp}_A} \subseteq S_{\mathrm{tp}_B}\).
\end{proof}

\subsection{Supporting Formulas for Corollary~\ref{corollary-1}}\label{sec:limitation}

In order to validate the theoretical results from corollary~\ref{corollary-1}, we select two methods: SDP-cert (proposed in \cite{raghunathan2018semidefinite}) and SDP-cert-loose which we propose. We begin by summarizing the formulation of SDP-cert (from \cite{raghunathan2018semidefinite} and formulate SDP-cert-loose by omitting certain constraints to replicate the setup proposed in Corollary~\ref{corollary-1}. In the last part, we summarize the findings of this experiment.

For an input \( x \in \mathbb{R}^d \) bounded by \( l \leq x \leq u \), the \( \ell_\infty \) perturbation model gives \( l = x - \varepsilon \mathbf{1} \), \( u = x + \varepsilon \mathbf{1} \), where \(\varepsilon\) represents the perturbation level and \(\mathbf{1}\) denote the vector of all ones. The input constraint can be expressed as a quadratic constraint (QC):
\begin{equation}\label{eq:input-relaxation-ps}
    (x - l)(x - u) \leq 0 \Rightarrow x^2 \leq (l + u) x - l u.
\end{equation}
ReLU constraints are captured by:
\begin{equation}\label{eq:relax-relu-ps}
z \geq 0, \quad z \geq x, \quad z(z - x) = 0,
\end{equation}
which together enforce \( z = \max(x, 0) \).

These constraints are lifted to an SDP by introducing the matrix \( P = vv^\top \), where \( v = [1 \; x^\top \; z^\top]^\top \). The resulting SDP relaxation (following~\cite{raghunathan2018semidefinite}) is:
\begin{align}\label{eq:sdp-relax-relu-ps}
	\begin{split}
            f_\text{SDP} = &\underset{P}{\text{max}} \ c^\top P[z] \\
            &\text{s.t.} \quad  P[z] \geq 0, \quad P[z] \geq W P[x]\\
            & \quad \quad\text{diag}(P[zz^\top]) = \text{diag}(WP[xz^\top])\\
            & \quad \quad \text{diag}(P[xx^\top]) \leq (l + u) \odot P[x] - l \odot u \\
            & \quad \quad P[1] = 1, \quad P \succeq 0.
	\end{split}
\end{align}
We refer to this SDP relaxation as \textbf{SDP-cert}. It provides an upper bound to the original non-convex verification problem in~\eqref{eq:verification-method}, allowing certified robustness via convex optimization.

\noindent\textbf{SDP-cert-loose.}
The ReLU constraints are written as a quadratic constraint: 
\begin{equation}\label{eq:relax-relu-1} 
z(z - x) = 0. 
\end{equation} 
The quadratic constraint enforces that \(z\) behaves as an overapproximation for ReLU activation. The input constraint remains unchanged in~\eqref{eq:input-relaxation-ps}. 
To obtain a convex relaxation, we reformulate the problem into an SDP, structurally similar to that proposed in~\cite{raghunathan2018semidefinite}, as follows:
\begin{align}\label{eq:sdp-relax-relu-1}
    \begin{split}
    f_\text{SDP}  = &\underset{P}{\text{max}} \ c^\top P[z] \\
    & \text{s.t.} \quad  \text{diag}(P[zz^\top]) = \text{diag}(WP[xz^\top]) \\
    & \quad \quad\text{diag}(P[xx^\top]) \leq (l + u) \odot P[x] - l \odot u \\
    & \quad \quad P[1] = 1, \quad P \succeq 0.
    \end{split}
\end{align}

\subsection{Fast Stepwise Relaxation (FSR)}\label{sec:appendix-fsr}
To justify the threshold choice in the FSR strategy, we present additional results in Tables~\ref{tab:FSR1}--\ref{tab:FSR3}. Each table reports the Loss, defined as the difference in the certified bound between submethods within the verifier. Across all experiments, this difference remains below 5\%. The selected threshold offers a favorable balance, enabling significant runtime reduction while preserving verification quality. This supports the adoption of early stopping when further tightening offers limited benefit.
\begin{table}[!h]
    \caption{Comparison of \(V_{13}\) and \(V_{12}\) on a single input (true class: 8) at \(\varepsilon = 0.2\).}
  \label{tab:FSR1}
  \centering
  \begin{tabular}{lcccccc}
    \toprule
     & \multicolumn{3}{c}{Optimal Answer} & \multicolumn{3}{c}{Runtime(min)}\\
    \cmidrule(r){1-1}\cmidrule(r){2-4}\cmidrule(r){5-7}
    Adv Class & \(V_{13}\) & \(V_{12}\) & Robust & \(V_{13}\) & \(V_{12}\) & Impv.(\%)\\
    \midrule
        0  & -0.127 & -0.127 & \cmark & 37.03  & 17.70 & 52.19\\
        1 & -0.084 & -0.084 & \cmark & 39.98  & 15.83 & 60.41\\
        2 & 0.264 & 0.264 & \xmark & 38.67 & 17.05 & 55.92\\
        3 & -2.8e-04 & -2.3e-04 & \cmark & 37.04 & 16.60 & 55.18\\
        4 & -0.107 & -0.107 & \cmark & 94.41 & 15.69  & 83.38\\
        5 & -0.313 & -0.313 & \cmark & 31.55 & 13.91 & 55.91\\
        6 & -0.244 & -0.244  & \cmark & 33.44 & 15.11 & 54.814\\
        7 & -0.178 & -0.178 & \cmark & 34.36 & 14.40 & 58.09\\
        9 & -0.124 & -0.124 & \cmark & 38.17  & 14.86 & 61.05\\
    \bottomrule
  \end{tabular}
\end{table}
\begin{table}[h!]
    \caption{Comparison of \(V_{13}\) and \(V_{12}\) on a single input (true class: 8) at \(\varepsilon = 0.1\).}
  \label{tab:FSR2}
  \centering
  \begin{tabular}{lcccccc}
    \toprule
     & \multicolumn{3}{c}{Optimal Answer} & \multicolumn{3}{c}{Runtime(min)}\\
    \cmidrule(r){1-1}\cmidrule(r){2-4}\cmidrule(r){5-7}
    Adv Class & \(V_{13}\) & \(V_{12}\) & Robust & \(V_{13}\) & \(V_{12}\) & Impv.(\%)\\
    \midrule
        0 & -0.587 & -0.587 & \cmark & 37.06 & 16.47 & 55.57\\
		1 & -0.469 & -0.469 & \cmark & 38.09  & 16.90 & 55.64\\
		2 & -0.179 & -0.179 & \cmark & 42.38 & 16.16 & 61.86\\
		3 & -0.436 & -0.436 & \cmark & 39.36 & 16.33 & 58.5\\
		4 & -0.563 & -0.563 & \cmark & 36.98 & 15.31 & 58.61\\
		5 & -0.785 & -0.785 & \cmark & 33.68 & 14.47 & 57.04\\
		6 & -0.696 & -0.696 & \cmark & 37.16 & 16.23 & 56.33\\
		7 & -0.573 & -0.572 & \cmark & 37.02 & 15.16 & 59.036\\
		9 & -0.589 & -0.589 & \cmark & 36.32 & 16.70 & 54.02\\
    \bottomrule
  \end{tabular}
\end{table}

\begin{table}[t!]
    \caption{Comparison of \(V_{13}\) and \(V_{12}\) on ten inputs (true class: 8) at \(\varepsilon = 0.2\), reporting mean \(\pm\) standard deviation.}
    \label{tab:FSR4}
    \centering
    \resizebox{0.38\textheight}{!}{
    \begin{tabular}{lcccc}
    \toprule
     & \multicolumn{1}{c}{Optimal Answer} & \multicolumn{2}{c}{Runtime(sec)}\\
    \cmidrule(r){1-1}\cmidrule(r){2-2}\cmidrule(r){3-5}
    Adv. Class& RA (\%) & \(V_{13}\) & \(V_{12}\) & Impv.(\%)\\
    \midrule
        0 & 77.78 & 2564.43 $\pm$ 53.03 & 1065.65 $\pm$ 28.19 & 58.44\\
		1 & 88.89 & 2456.32 $\pm$ 77.18 & 1063.89 $\pm$ 37.15 & 56.69\\
		2 & 77.78 & 2496.45 $\pm$ 100.90 & 1048.79 $\pm$ 22.62 & 57.99\\
		3 & 77.78 & 2500.48 $\pm$ 96.30 & 1063.72 $\pm$ 41.93 & 57.46\\
		4 & 77.78 & 2407.48 $\pm$ 62.24& 1039.28 $\pm$ 27.59 & 56.83\\
		5 & 66.67 & 2518.46 $\pm$ 111.58 & 1065.92 $\pm$ 49.64 & 57.68\\
		6 & 77.78 & 2401.30 $\pm$ 82.29& 1043.59 $\pm$ 30.04 & 56.54\\
		7 & 66.67 & 2478.21 $\pm$ 127.22& 1067.22$\pm$ 50.30 & 56.93\\
		9 & 88.89 & 2574.17 $\pm$ 100.93 & 1082.72$\pm$ 42.65 & 57.94\\
    \bottomrule
    \end{tabular}}
\end{table}

\begin{table}[h!]
    \caption{Comparison of \(V_{13}\) and \(V_{12}\) on a single input (true class: 0) at \(\varepsilon = 0.1\).}
  \label{tab:FSR3}
  \centering
  \begin{tabular}{lcccccc}
    \toprule
    True Class-0 & \multicolumn{3}{c}{Optimal Answer} & \multicolumn{3}{c}{Runtime(min)}\\
    \cmidrule(r){1-1}\cmidrule(r){2-4}\cmidrule(r){5-7}
    Adv Class & \(V_{13}\) & \(V_{12}\) & Robust & \(V_{13}\) & \(V_{12}\) & Impv.(\%)\\
    \midrule
        1 & -1.219 & -1.219 & \cmark & 36.19 & 16.71 & 53.82\\
		2 & -1.073 & -1.073 & \cmark & 37.05 & 16.15 & 56.4\\
		3 & -1.171 & -1.171 & \cmark & 39.14 & 18.57 & 52.54\\
		4  & -1.026 & -1.026 & \cmark & 34.82 & 20.14 &  42.15\\
		5  & -0.945 & -0.945 & \cmark & 44.62 & 18.44 & 58.68\\
		6  & -1.048 & -1.048 & \cmark & 38.64 & 21.07 &  45.47\\
		7  & -1.232 & -1.232 & \cmark & 35.12 & 18.38 & 47.65\\
		8  & -1.103 & -1.103 & \cmark & 36.92 & 20.52 & 44.41\\
		9 & -1.626 & -1.626 & \cmark & 42.66 & 16.67 & 60.93\\
    \bottomrule
  \end{tabular}
\end{table}
\subsection{Additional Results}\label{sec:appendix-experiemt-eps}
 Tables~\ref{tab:SWR-epsilon} and~\ref{tab:CRV-epsilon} present detailed certification results for varying perturbation budgets (\(\varepsilon = \{0.15, 0.2, 0.25\}\)) on MNIST. We report robust accuracy (RA), average runtime per sample, and speedup relative to baseline SDP-based methods. These results confirm that CRV-SR and CRV-FSR consistently maintain high robust accuracy while achieving substantial runtime reductions across different verification settings.
\begin{table}[!h]
  \caption{Certified robustness on MNIST under \(l_{\infty}\) perturbations at \(\varepsilon = \{0.15, 0.2, 0.25\}\). RA, average runtime per sample, and speedup (relative to \(V_{13}\)) are reported.}
  \label{tab:SWR-epsilon}
  \centering
  \begin{tabular}{llccc}
    \toprule
    \( \varepsilon \) & Verifier & RA & Runtime (min) & Speedup\\
    \midrule
    \multirow{5}{*}{0.15} 
        & \(V_{11}\) & 0\% & 25.57 & -\\
        & \(V_{12}\) & 84\% & 144.11 & -\\
        & \(V_{13}\)~\cite{raghunathan2018semidefinite} & 84\% & 318.27 & -\\
        & SR & 84\% & 220.60 & 30.69\%\\
        & \textbf{FSR} & \textbf{84\%} & \textbf{169.68} & \textbf{46.69\%}\\
    \midrule
    \multirow{5}{*}{0.2} 
        & \(V_{11}\) & 0\% & 18.26 & -\\
        & \(V_{12}\) & 66\% & 121.40 & -\\
        & \(V_{13}\)~\cite{raghunathan2018semidefinite} & 66\% & 295.75 & -\\
        & SR &  66\% & 240.22 & 18.78\%\\
        & \textbf{FSR} &  \textbf{66\%} & \textbf{139.66} & \textbf{52.78\%}\\
    \midrule
    \multirow{5}{*}{0.25}
        & \(V_{11}\) & 0\% & 17.92 & -\\
        & \(V_{12}\) & 22\% & 66.24 & -\\
        & \(V_{13}\)~\cite{raghunathan2018semidefinite} & 22\% & 155.89 & -\\
        & SR & 22\% & 205.75 & 31.98\%\\
        & \textbf{FSR} & \textbf{22\%} & \textbf{84.16} & \textbf{46.01\%}\\
    \bottomrule
  \end{tabular}
\end{table}
\begin{table}[!t]
  \caption{Certification results on MNIST under \(l_{\infty}\)-norm attacks at \(\varepsilon = \{0.15, 0.2, 0.25\}\). RA, average runtime per sample, and speedup (relative to SDP-cert) are reported. LP-cert is fast and, since its runtime (per sample) is significantly lower than SDP-based methods and not reported in prior work~\cite{wong2018scaling, dathathri2020enabling, li2023sok}, we omit it from the runtime and speedup analysis.}
  \label{tab:CRV-epsilon}
  \centering
  \resizebox{0.48\textwidth}{!}{
  \begin{tabular}{llcccccc}
    \toprule
     &  & \multicolumn{3}{c}{Grad-NN} & \multicolumn{3}{c}{LP-NN}\\
    \cmidrule(r){1-1} \cmidrule(r){2-2} \cmidrule(r){3-5} \cmidrule(r){6-8}
    \(\varepsilon\) & Verifier & RA & Runtime (min) & Speedup & RA & Runtime (min) & Speedup\\
    \midrule
        \multirow{5}{*}{0.15}
        & SDP-cert~\cite{raghunathan2018semidefinite}  & 84\% & 318.27 & - & 84\% & 318.27 & - \\
        & LP-cert~\cite{wong2018provable}  & 14\% & - & - & 80\% & - & - \\
        & CRV  & 84\% & 273.71 & 14.00\% & 84\% & 63.65 & 80.00\% \\
        & CRV-SR & 84\% & 189.72 & 40.39\% & 84\% & 44.12 & 86.14\% \\
        & \textbf{CRV-FSR} & \textbf{84\%} & \textbf{145.92} & \textbf{54.15\%} & \textbf{84\%} & \textbf{33.94} & \textbf{89.34\%} \\
        & PGD Success
        & 16\%  & - & - & 12\%  & - & - \\
        \midrule
        \multirow{5}{*}{0.2}
        & SDP-cert~\cite{raghunathan2018semidefinite}  & 66\% & 295.75 & - & 66\% & 295.75 & - \\
        & LP-cert~\cite{wong2018provable}  & 14\% & - & - & 78\% & - & - \\
        & CRV  & 68\% & 254.34 & 14.00\% & 82\% & 65.06 & 78.00\% \\
        & CRV-SR & 68\% & 206.59 & 30.15\% & 82\% & 52.85 & 82.13\% \\
        & \textbf{CRV-FSR} & \textbf{68\%} & \textbf{120.11} & \textbf{59.39\%} & \textbf{82\%} & \textbf{30.72} & \textbf{89.61\%} \\
        & PGD Success
        & 28\%  & - & - & 16\%  & - & - \\
        \midrule
        \multirow{5}{*}{0.25}
        & SDP-cert~\cite{raghunathan2018semidefinite}  & 22\% & 155.89 & - & 22\% & 155.89 & - \\
        & LP-cert~\cite{wong2018provable}  & 14\% & - & - & 78\% & - & - \\
        & CRV & 36\% & 134.06 & 8.23\% & 78\% & 34.30 & 78.00\% \\
        & CRV-SR & 36\% & 176.94 & -13.50\% & 78\% & 45.26 & 70.97\% \\
        & \textbf{CRV-FSR} & \textbf{36\%} & \textbf{72.38} & \textbf{53.57\%} & \textbf{78\%} & \textbf{18.52} & \textbf{88.12\%} \\
        & PGD Success
        & 58\%  & -      & -     & 22\%  & -      & -     \\
        \bottomrule
    \end{tabular}}
\end{table}

\end{document}